\documentclass{article}
\usepackage[preprint]{neurips_2026}

\usepackage[utf8]{inputenc}
\usepackage[T1]{fontenc}
\usepackage{amsmath,amssymb,amsthm}
\usepackage{booktabs}
\usepackage{multirow}
\usepackage{graphicx}
\usepackage{caption}
\usepackage{tikz}
\usepackage{url}
\usepackage{hyperref}
\usepackage{microtype}
\usepackage{xcolor}
\usepackage{enumitem}
\usepackage{fontawesome5}
\usepackage{hyperref}
\usepackage{xcolor}

\definecolor{githubblue}{HTML}{0366D6}

\hypersetup{
    colorlinks=true,
    urlcolor=githubblue,
}

\newtheorem{proposition}{Proposition}

\title{Decoupling KL and Trajectories: A Unified Perspective for SFT, DAgger, Offline RL, and OPD in LLM Distillation}

\author{
\makebox[\textwidth][c]{\textbf{Anhao Zhao\textsuperscript{1,2}\quad
Haoran Xin\textsuperscript{4}\quad
Yingqi Fan\textsuperscript{1}\quad
Junlong Tong\textsuperscript{1,3}\quad
Wenjie Li\textsuperscript{2}\quad
Xiaoyu Shen\textsuperscript{1}\thanks{Corresponding Author}}}\\
\normalfont
\makebox[\textwidth][c]{\textsuperscript{1}Eastern Institute of Technology, Ningbo\quad
\textsuperscript{2}The Hong Kong Polytechnic University}\\[-0.2ex]
\makebox[\textwidth][c]{\textsuperscript{3}Shanghai Jiao Tong University\quad
\textsuperscript{4}Thrust of Artificial Intelligence,}\\[-0.2ex]
\makebox[\textwidth][c]{The Hong Kong University of Science and Technology (Guangzhou)}\\[-0.2ex]
\makebox[\textwidth][c]{\texttt{anhao.zhao@connect.polyu.hk}\quad
\texttt{xyshen@eitech.edu.cn}}
}

\begin{document}

\maketitle

\begin{abstract}
Knowledge distillation has become central to LLM post-training, yet its design
space remains poorly understood, especially alongside reinforcement learning (RL). We show that the prevailing paradigms, off-policy distillation
and on-policy distillation (OPD), implicitly \textbf{couple two orthogonal choices:
prefix source and token-level KL direction}. This coupling follows from
decomposing sequence-level KL over autoregressive response distributions:
\textit{forward KL pairs teacher prefixes with token-level forward KL, and reverse KL pairs student prefixes with token-level reverse KL}. We argue that this
coupling is not intrinsic: decoupling the two axes yields four valid objectives. We establish gradient-level identities showing that
forward KL gives SFT-style cross-entropy matching with teacher soft targets,
whereas reverse KL gives an RL-style policy-gradient objective with a dense
teacher-student log-ratio reward, connecting the four objectives to off-policy
SFT, DAgger-style on-policy SFT, offline-RL-style distillation, and OPD.
We conduct an extensive controlled study on math reasoning,
evaluating the four objectives both as standalone distillation methods and as
initializations for subsequent RL. The results reveal three tradeoffs: \textit{KL
direction induces an accuracy--entropy tradeoff, prefix source induces a
quality--compute tradeoff, and training length induces an accuracy--stability
tradeoff}. Motivated by these findings, we propose \textbf{KL mixing} and an
\textbf{entropy-gated length curriculum}. KL mixing shows that long-sequence
distillation requires substantial forward-KL weight to prevent entropy collapse
and length inflation without sacrificing accuracy. The entropy-gated length curriculum improves Avg@k and Pass@k by \textbf{3.6} and up to \textbf{5.8} points, and reduces average response length by roughly $\boldsymbol{3\times}$ relative to fixed long-horizon training.
Together, our results provide a framework and practical methods for designing reasoning distillation objectives that balance accuracy, diversity, compute, and RL behavior.
We release our code at \textcolor{githubblue}{\faGithub}\ \href{https://github.com/EIT-NLP/Decoupled-Distill}{\textbf{\texttt{EIT-NLP/Decoupled-Distill}}}.
\end{abstract}
\section{Introduction}
\label{introduction}

The capabilities of Large Language Models (LLMs) are increasingly shaped by post-training, where knowledge distillation~\citep{hinton2015kd} and reinforcement learning (RL) have together become the central recipe for building frontier reasoning models~\citep{deepseek-r1, qwen3}. Modern pipelines typically follow one of two paradigms: off-policy distillation on teacher-generated traces followed by RL~\citep{deepseek-r1, muennighoff2025s1, openthoughts}, or on-policy distillation (OPD) interleaved with RL~\citep{qwen3, mimo-v2, glm5, qwen3.5_omni}. Despite the prevalence of these recipes, the design space of distillation, and how its choices interact with RL, remains poorly understood.

\textbf{Off-policy distillation} aligns the student's token-level output distribution with the teacher's along teacher-generated prefixes. In practice, this is often implemented as supervised fine-tuning (SFT) on teacher-generated traces, which treats sampled teacher outputs as hard labels and can be viewed as a Monte Carlo approximation to distillation from the teacher's output distribution. This recipe has been widely adopted in recent reasoning-model pipelines, including DeepSeek-R1~\citep{deepseek-r1}, s1~\citep{muennighoff2025s1}, and OpenThinker~\citep{openthoughts}.
\textbf{On-policy distillation} (OPD) instead lets the student generate its own rollouts at training time and uses the teacher's per-token log-probabilities as a dense supervision signal on the states the student visits~\citep{gu2023minillm, agarwal2023gkd, lu2025opd}. By supervising student-visited states, OPD mitigates exposure bias~\citep{bengio2015scheduled}, while its dense token-level feedback complements the sparse outcome rewards of RL. OPD has been incorporated into the post-training pipelines of Qwen3~\citep{qwen3}, MiMo~\citep{mimo-v2}, and GLM-5~\citep{glm5}.

Despite their differences, the two prevailing paradigms share the same implicit
structure: they couple prefix source with token-level KL direction. Off-policy
distillation uses teacher-generated prefixes with forward-KL supervision,
whereas OPD uses student-generated prefixes with reverse-KL
supervision. This pairing follows directly from sequence-level KL: by the
autoregressive chain rule, $\mathrm{KL}(\pi_T\|\pi_S)$ decomposes into
token-level forward KL on teacher prefixes, while $\mathrm{KL}(\pi_S\|\pi_T)$
decomposes into token-level reverse KL on student prefixes. We argue that \textbf{this
coupling is not intrinsic} at the token level: prefix source specifies where
supervision is applied, while KL direction specifies how teacher and student
token distributions are compared. Taking their Cartesian
product yields four valid objectives. The two off-diagonal objectives, namely
teacher prefixes with reverse KL and student prefixes with forward KL, have not,
to our knowledge, been systematically studied.

\begin{figure}
    \centering
    \includegraphics[width=\linewidth]{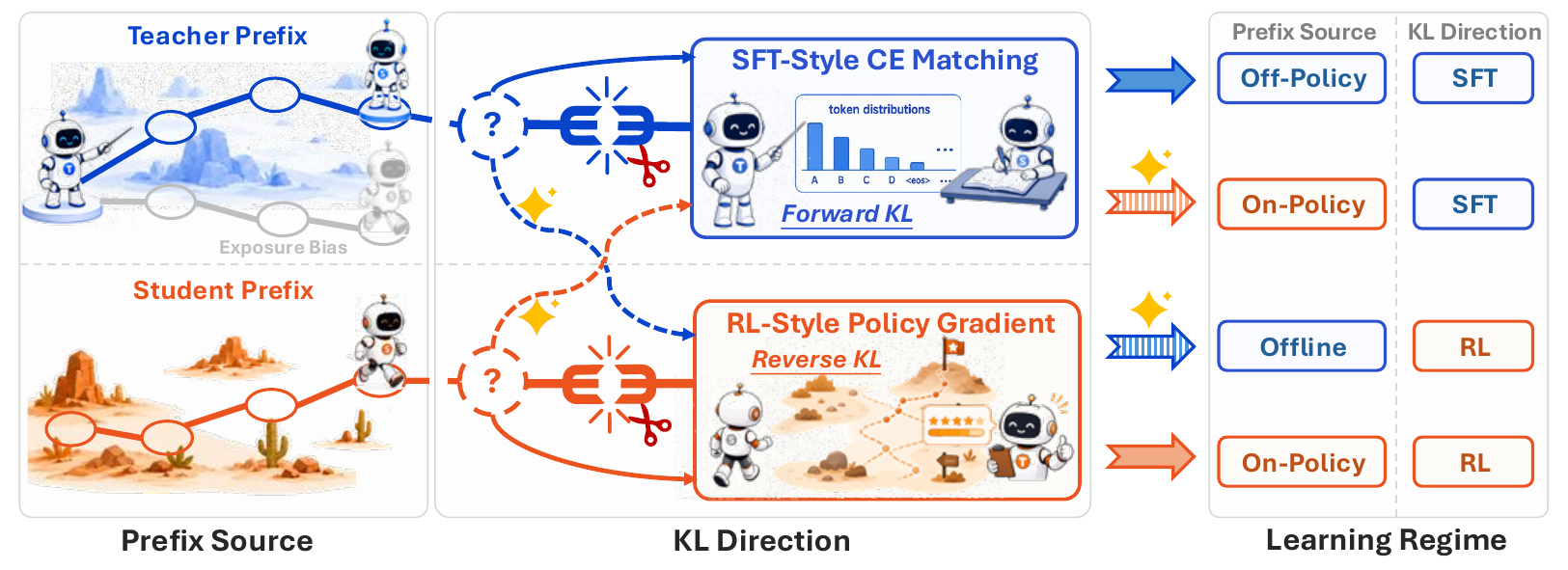}
    \vspace{-20pt}
    \caption{\small Overview of our decoupled distillation framework. Conventional objectives couple prefix source with KL direction; decoupling these axes yields four objectives that correspond to classical training regimes.}
    \label{fig:illustration}
\end{figure}

From a theoretical perspective, this orthogonal decomposition connects the four
objectives to well-established training paradigms, , as summarized in
Figure~\ref{fig:illustration}. Prefix
source determines the policy regime: \emph{teacher prefixes are
off-policy} for the student, as they are generated by another policy, whereas
\emph{student prefixes are on-policy}. KL direction determines the learning objective: we establish that
forward KL gives the gradient of SFT-style cross-entropy matching with teacher
soft targets, while reverse KL gives an RL-style policy-gradient objective with
a dense teacher--student log-ratio reward. Combining these correspondences
interprets the \textbf{four objectives as classical training regimes}: teacher-prefix forward KL
instantiates off-policy SFT \citep{deepseek-r1}; student-prefix forward KL matches
DAgger-style on-policy SFT \citep{ross2011dagger}; teacher-prefix reverse KL realizes an
offline-RL-style distillation objective \citep{levine2020offlinerl}; and student-prefix reverse KL yields OPD, viewed as dense-reward on-policy RL \citep{lu2025opd}.
Our theoretical framework therefore unifies training recipes often studied
separately as different instantiations of the same two design choices in
autoregressive distillation.

Guided by this lens, we comprehensively investigate the four decoupled objectives with Qwen3-4B/8B
teachers and a Qwen3-0.6B student on mathematical reasoning. We track accuracy,
predictive entropy, and response length during training, and evaluate Avg@$k$,
Pass@$k$, and response length on AIME24, AMC23, MATH500, and GSM8K. Each objective is studied both as standalone distillation and
as an initialization for subsequent RL, revealing three intertwined tradeoffs.
First, \textbf{KL direction induces an
accuracy--entropy tradeoff}: reverse KL
improves Avg@$k$ but sharpens the student distribution, reducing diversity,
weakening Pass@$k$, and making downstream RL less reliable; forward KL preserves
entropy and supports more stable RL improvement. Second, \textbf{prefix source induces a quality--compute tradeoff}: student prefixes perform better under
matched training steps by supervising student-visited states, whereas teacher
prefixes can be more compute-effective under matched FLOPs by reusing offline
trajectories and cached teacher logits.
Third, \textbf{training length induces an
accuracy--stability tradeoff}: long-sequence distillation improves reasoning accuracy but, under reverse KL,
can drive entropy close to zero and cause severe length inflation;
short-sequence distillation is more stable but less accurate. These findings show that the best standalone
distillation objective is not necessarily the best objective for a
distillation-then-RL pipeline, and that effective reasoning post-training must
balance accuracy, diversity, compute, and downstream RL behavior.

To navigate these tradeoffs, we propose two targeted methods. For the KL-direction tradeoff, we introduce \textbf{KL mixing}, which forms a weighted combination of forward and reverse token-level KL. In long-sequence distillation, we find that a high forward-KL weight is crucial: it prevents entropy collapse and length inflation with little or no accuracy loss, while high reverse-KL mixtures remain unstable. For the training-length tradeoff, we introduce an \textbf{entropy-gated length curriculum}, which starts from a short training horizon and increases the length only while predictive entropy remains above a stability threshold. Compared with fixed 4096-token training, it improves Avg@$k$ by $\textbf{3.6}$ points, raises Pass@$k$ by up to $\textbf{5.8}$ points, and reduces average response length by roughly $\boldsymbol{3\times}$. Together, these methods turn the decoupled view from an explanatory framework into a practical tool for balancing accuracy, diversity, and generation stability.


\section{Disentangling KL Direction and Prefix Source}
\label{sec:framework}

\subsection{Problem Setup: Sequence-Level KL for Reasoning Distillation}
\label{sec:problem-definition}

We study mathematical reasoning, where an LLM maps a prompt $x$ to a response
$y=(y_1,\ldots,y_L)$ consisting of a reasoning trace and a final answer.
Reasoning distillation transfers this behavior from a strong teacher to
a trainable student by matching their response distributions
$p_T(y\mid x)$ and $q_\theta(y\mid x)$~\citep{hinton2015kd,deepseek-r1}. At the
sequence level, this matching can use either forward or reverse
KL~\citep{sequence_level_kd,gu2023minillm}:
\[
\mathrm{KL}(p_T \| q_\theta)(x)
=
\mathbb{E}_{y\sim p_T(\cdot\mid x)}
\!\left[
\log \frac{p_T(y\mid x)}{q_\theta(y\mid x)}
\right],
\quad
\mathrm{KL}(q_\theta \| p_T)(x)
=
\mathbb{E}_{y\sim q_\theta(\cdot\mid x)}
\!\left[
\log \frac{q_\theta(y\mid x)}{p_T(y\mid x)}
\right].
\]
Forward KL averages over teacher responses, whereas reverse KL averages over
student responses; the two objectives also use opposite teacher-student
log-ratio directions.

\subsection{From Coupled Sequence-Level KL to Decoupled Token-Level Objectives}
\label{sec:sequence-kl-coupling}

Teacher and student response distributions factorize autoregressively:
\begin{equation}
\label{eq:ar-factorization}
p_T(y\mid x)=\prod_{t=1}^{L}p_T(y_t\mid x,y_{<t}),
\qquad
q_\theta(y\mid x)=\prod_{t=1}^{L}q_\theta(y_t\mid x,y_{<t}).
\end{equation}
We write $s_t=(x,y_{<t})$ for the prefix, or generation state, at step $t$.
Let $d_T^t$ and $d_S^t$ denote the distributions over prefixes induced by
teacher-generated and student-generated responses, respectively.
Applying the autoregressive factorization to the sequence-level KL objectives
gives
\begin{equation}
\label{eq:fwd-kl}
\mathrm{KL}(p_T\|q_\theta)(x)
=
\sum_{t=1}^{L}
\mathbb{E}_{s_t\sim d_T^t}
\!\left[
\mathrm{KL}(p_T(\cdot\mid s_t)\|q_\theta(\cdot\mid s_t))
=
\mathbb{E}_{a\sim p_T(\cdot\mid s_t)}
\!\left[
\log \frac{p_T(a\mid s_t)}{q_\theta(a\mid s_t)}
\right]
\right],
\end{equation}
\begin{equation}
\label{eq:rev-kl}
\mathrm{KL}(q_\theta\|p_T)(x)
=
\sum_{t=1}^{L}
\mathbb{E}_{s_t\sim d_S^t}
\!\left[
\mathrm{KL}(q_\theta(\cdot\mid s_t)\|p_T(\cdot\mid s_t))
=
\mathbb{E}_{a\sim q_\theta(\cdot\mid s_t)}
\!\left[
\log \frac{q_\theta(a\mid s_t)}{p_T(a\mid s_t)}
\right]
\right].
\end{equation}
where $a$ denotes the next token.
\textbf{These decompositions reveal an implicit coupling in sequence-level KL
distillation}: forward sequence KL pairs teacher-induced prefixes with
token-level forward KL, whereas reverse sequence KL pairs student-induced
prefixes with token-level reverse KL. This coupling is mirrored in current
post-training pipeline. Off-policy distillation trains on fixed
teacher-generated traces and is commonly implemented with SFT, where sampled
teacher tokens serve as hard labels, a Monte Carlo approximation to token-level
forward-KL matching~\citep{deepseek-r1,muennighoff2025s1,openthoughts}.
In contrast, on-policy distillation samples prefixes from the current student
and uses teacher feedback through reverse-KL matching on student-visited
states~\citep{gu2023minillm,agarwal2023gkd,lu2025opd}.

The sequence-level derivation explains the conventional pairings, but
\textbf{the coupling is not intrinsic}: KL direction and prefix source can be
chosen independently. KL direction specifies how teacher and student token
distributions are compared at a prefix, using either
$\mathrm{KL}(p_T(\cdot\mid s_t)\|q_\theta(\cdot\mid s_t))$ or
$\mathrm{KL}(q_\theta(\cdot\mid s_t)\|p_T(\cdot\mid s_t))$; prefix source
specifies where the loss is evaluated, with prefixes drawn from $d_T^t$ or
$d_S^t$. Their Cartesian product yields four token-level distillation
objectives. Existing practice primarily uses the two sequence-level pairings,
teacher prefixes with forward KL and student prefixes with reverse KL. The two
off-diagonal objectives, teacher prefixes with reverse KL and student prefixes
with forward KL, are equally well-defined but have not been systematically
studied.


\subsection{Learning-Regime Interpretation of the Decoupled Objectives}
\label{sec:gradient-equivalence}

\textbf{The prefix-source axis determines the policy regime}: teacher-induced
prefixes are off-policy for the student, whereas student-induced prefixes are
on-policy. \textbf{The KL-direction axis determines the learning objective}:
Proposition~\ref{prop:sl-rl} shows that forward KL gives the gradient of
SFT-style cross-entropy matching with teacher soft targets, while reverse KL
gives an RL-style policy-gradient objective with a dense teacher--student
log-ratio reward.

\begin{proposition}[Token-level KL gradients]
\label{prop:sl-rl}
Fix a prefix $s_t$ and treat $p_T(\cdot\mid s_t)$ as independent of $\theta$.
Let $q_\theta(\cdot\mid s_t)$ be differentiable. Then:
\begin{enumerate}[label=\emph{(\roman*)}, leftmargin=1.5em, itemsep=0.35em]
  \item The forward token-level KL satisfies
  \[
    \nabla_\theta
    \mathrm{KL}\!\bigl(p_T(\cdot\mid s_t)\,\|\,q_\theta(\cdot\mid s_t)\bigr)
    =
    -\,\mathbb{E}_{y\sim p_T(\cdot\mid s_t)}
    \!\left[\nabla_\theta \log q_\theta(y\mid s_t)\right].
  \]
  Thus, forward KL is gradient-equivalent to cross-entropy matching with
  teacher soft targets; SFT on teacher-sampled tokens is its Monte Carlo
  hard-label form.

  \item The reverse token-level KL satisfies
  \[
    \nabla_\theta
    \mathrm{KL}\!\bigl(q_\theta(\cdot\mid s_t)\,\|\,p_T(\cdot\mid s_t)\bigr)
    =
    \mathbb{E}_{y\sim q_\theta(\cdot\mid s_t)}
    \!\left[
      \bigl(\log q_\theta(y\mid s_t)-\log p_T(y\mid s_t)\bigr)
      \nabla_\theta \log q_\theta(y\mid s_t)
    \right].
  \]
  Hence, minimizing reverse KL gives a REINFORCE-style ascent direction with
  dense reward
  \[
    r(s_t,y)=\log p_T(y\mid s_t)-\log q_\theta(y\mid s_t),
  \]
  treating $r$ as scalar feedback, i.e., stopping gradients through $r$.
\end{enumerate}
\end{proposition}

The proof is provided in Appendix~\ref{app:proof-sl-rl}.
Combining the prefix-source axis, which determines the policy regime, with the
KL-direction axis, which determines the gradient form of the token-level
objective, gives a learning-theoretic interpretation of all four decoupled
objectives. Teacher-prefix forward KL is off-policy SFT \citep{deepseek-r1}, or soft-label
distillation, on teacher traces. Student-prefix forward KL is DAgger-style on-policy SFT \citep{ross2011dagger}, where the student
visits states and the teacher provides SFT-style cross-entropy supervision. Teacher-prefix reverse KL is an offline-RL-style
distillation objective \citep{levine2020offlinerl}, applying the log-ratio reward on a fixed
teacher-induced prefix distribution. Student-prefix reverse KL yields OPD \citep{gu2023minillm},
viewed as dense-reward on-policy RL induced by the teacher--student log-ratio.

After defining the four decoupled objectives and their learning-theoretic
interpretations, we evaluate how prefix source and KL direction affect reasoning
distillation in two regimes. In \textbf{standalone distillation}, all objectives
use matched training steps, optimizer, and hyperparameters, and we compare the
quality of the resulting students. In a \textbf{distillation-then-RL} pipeline,
each distilled checkpoint initializes a subsequent RL stage, testing whether
strong standalone distillation also yields an effective starting point for
further policy optimization.

\section{Experimental Setup}
\label{sec:setup_exp}

\paragraph{Models.}
We distill Qwen3-4B and Qwen3-8B into Qwen3-0.6B-Base~\citep{qwen3}, using the
same model family to avoid cross-tokenizer artifacts in token-level KL
computation.


\paragraph{Training.}
For \textbf{standalone distillation}, we study two training lengths: a short
setting with $128$ tokens and a long setting with $4096$ tokens. All objectives
use bf16 training, learning rate $5\times10^{-7}$, batch size $32$, and $1000$
training steps. For the \textbf{RL follow-up}, we use Group Relative Policy
Optimization (GRPO)~\citep{shao2024grpo} with an accuracy-based outcome reward.
We use group size $8$, batch size $32$, rollout temperature $1.0$,
top-$p=0.95$, maximum decoding length $4096$, and $1000$ steps. We report
training dynamics averaged over three runs.

\paragraph{Data and evaluation.}
We train on DeepScaleR~\citep{deepscaler2025} and
evaluate on AIME24, AMC23, MATH500, and GSM8K~\citep{aime24,AMC2023,math500,
gsm8k}. Math evaluation uses temperature $0.6$, top-$p=0.95$, and a maximum
generation length of $8192$ tokens. We report Avg@$N$, Pass@$N$, and average
response length, with $N=5$ for AIME24/AMC23 and $N=3$ for MATH500/GSM8K.

\paragraph{Training dynamics.}
We track task accuracy,
mean per-token predictive entropy, and average response length as diagnostics of
performance, exploration capacity, and generation behavior.

\paragraph{Fused full-vocabulary KL kernel.}
To make exact full-vocabulary KL feasible, we implement a custom fused kernel
that avoids materializing vocabulary-sized intermediates at each token
position.\footnote{
The kernel fuses projection, softmax, and KL accumulation in a streaming
vocabulary pass, preserving the exact objective while reducing per-token memory
from $O(|\mathcal{V}|)$ intermediates to $O(1)$ running statistics. Details are
in Appendix~\ref{app:fused-kl}.}
\begin{figure}[!t]
\centering
\begin{minipage}{0.32\textwidth}
  \centering
  \includegraphics[width=\textwidth]{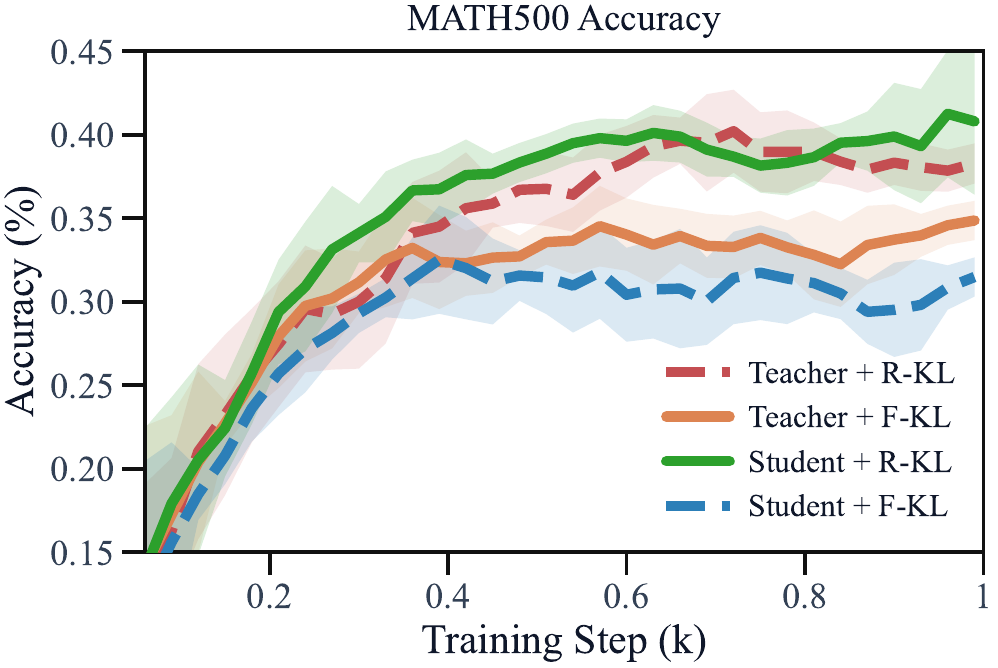}
\end{minipage}
\hfill
\begin{minipage}{0.32\textwidth}
  \centering
  \includegraphics[width=\textwidth]{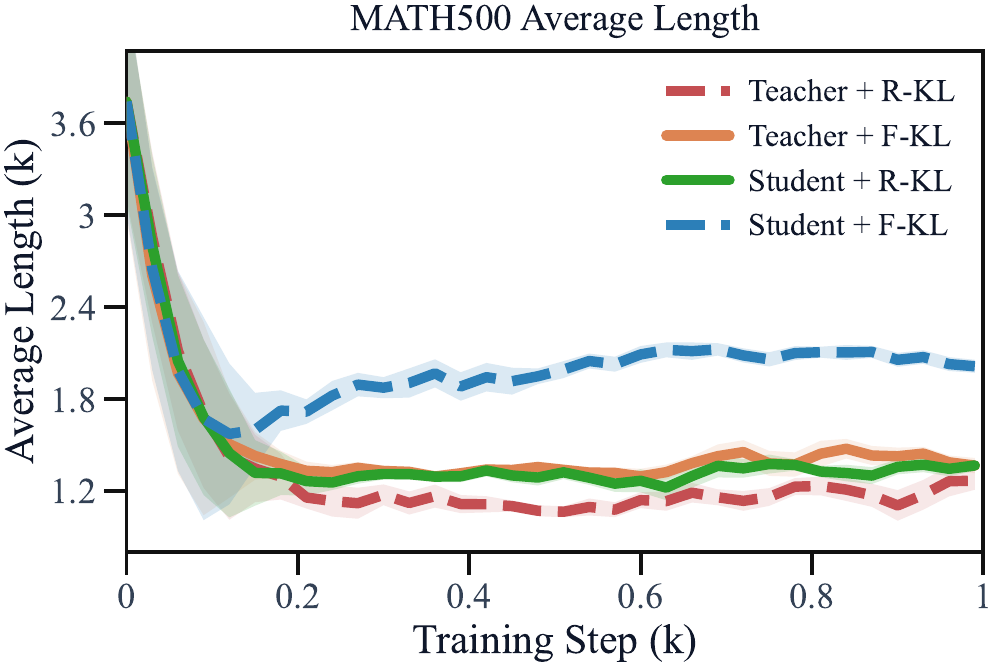}
\end{minipage}
\hfill
\begin{minipage}{0.32\textwidth}
  \centering
  \includegraphics[width=\textwidth]{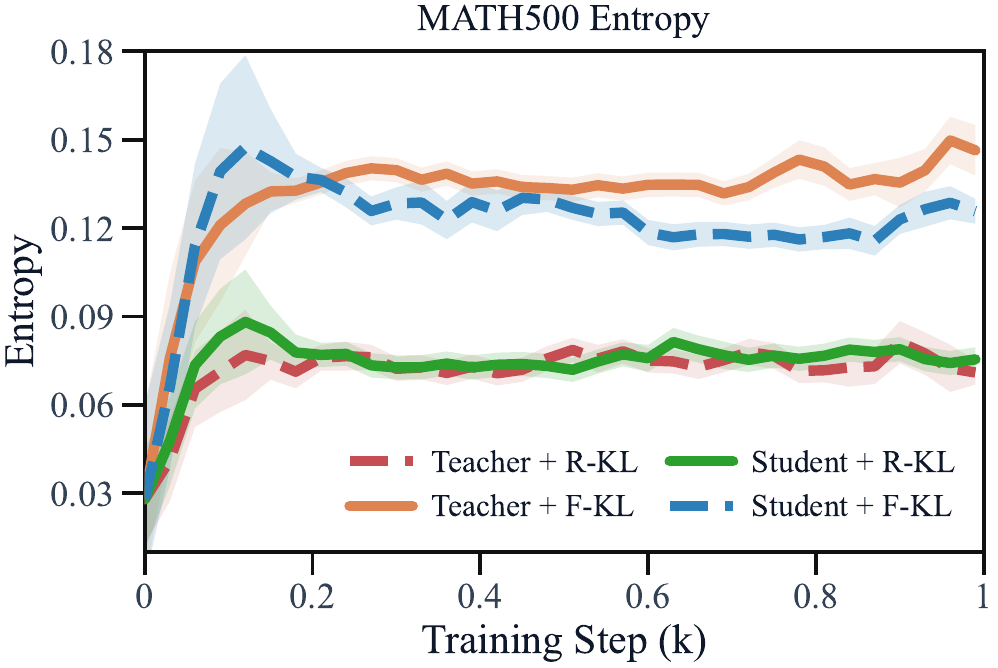}
\end{minipage}


\begin{minipage}{0.32\textwidth}
  \centering
  \includegraphics[width=\textwidth]{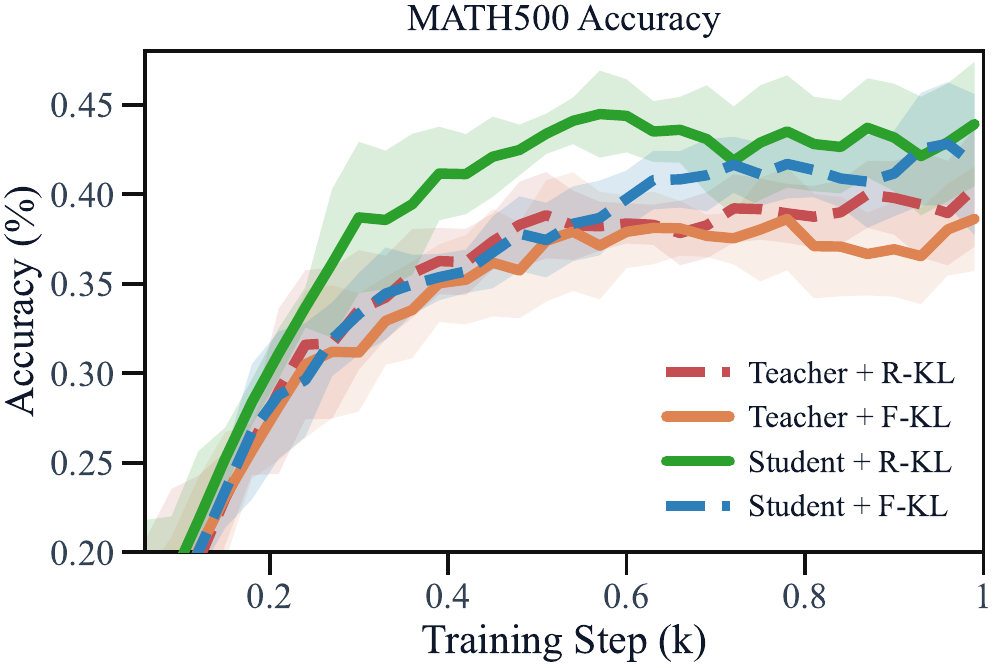}
\end{minipage}
\hfill
\begin{minipage}{0.32\textwidth}
  \centering
  \includegraphics[width=\textwidth]{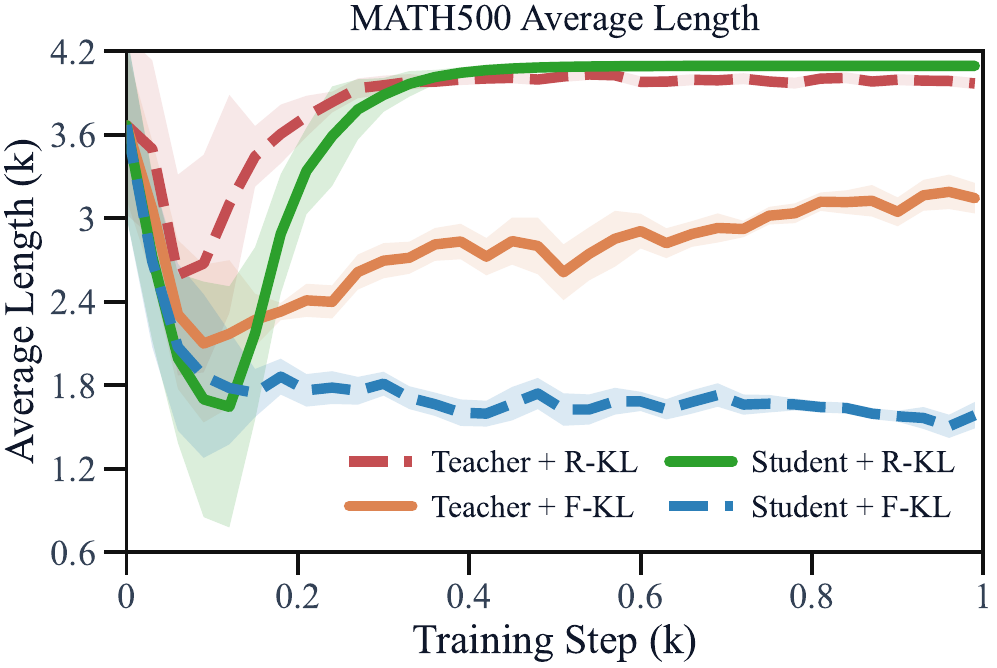}
\end{minipage}
\hfill
\begin{minipage}{0.32\textwidth}
  \centering
  \includegraphics[width=\textwidth]{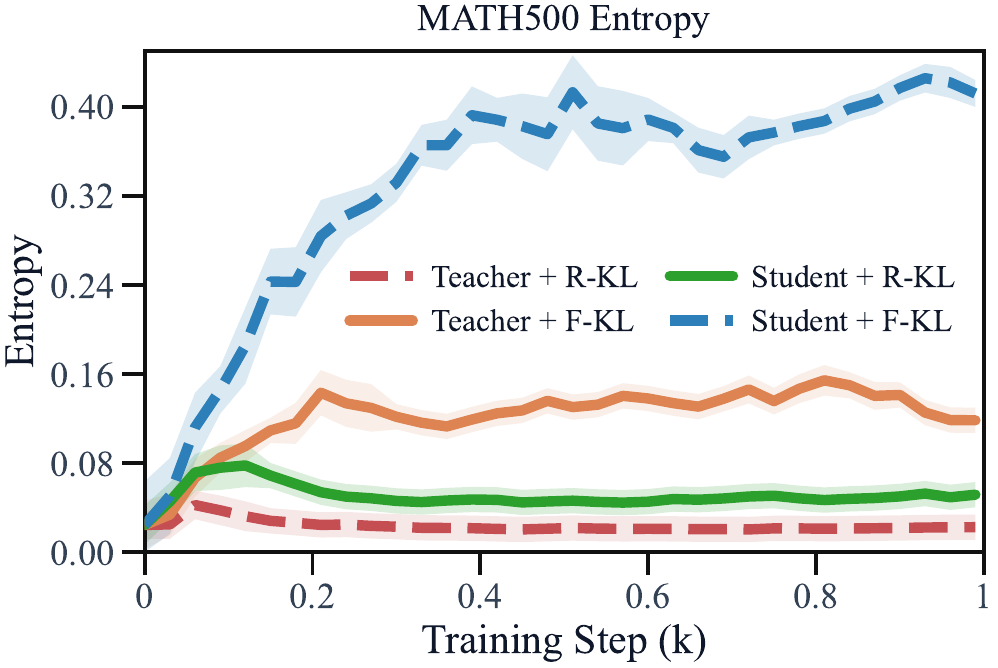}
\end{minipage}
\vspace{-4pt}
\caption{Distillation training dynamics with Qwen3-4B as teacher and
Qwen3-0.6B as student. Top/bottom rows: 128/4096-token training; left-to-right
columns: accuracy, length, entropy.}
\label{fig:math500_dynamics}
\end{figure}

\begin{figure}[!t]
\centering
\begin{minipage}{0.32\textwidth}
  \centering
  \includegraphics[width=\textwidth]{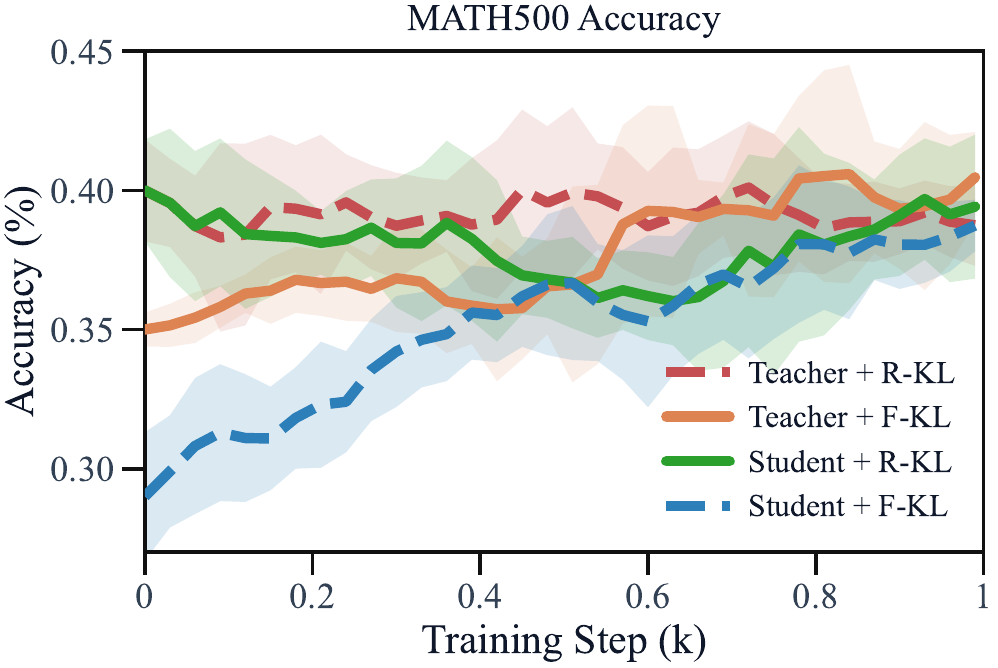}
\end{minipage}
\hfill
\begin{minipage}{0.32\textwidth}
  \centering
  \includegraphics[width=\textwidth]{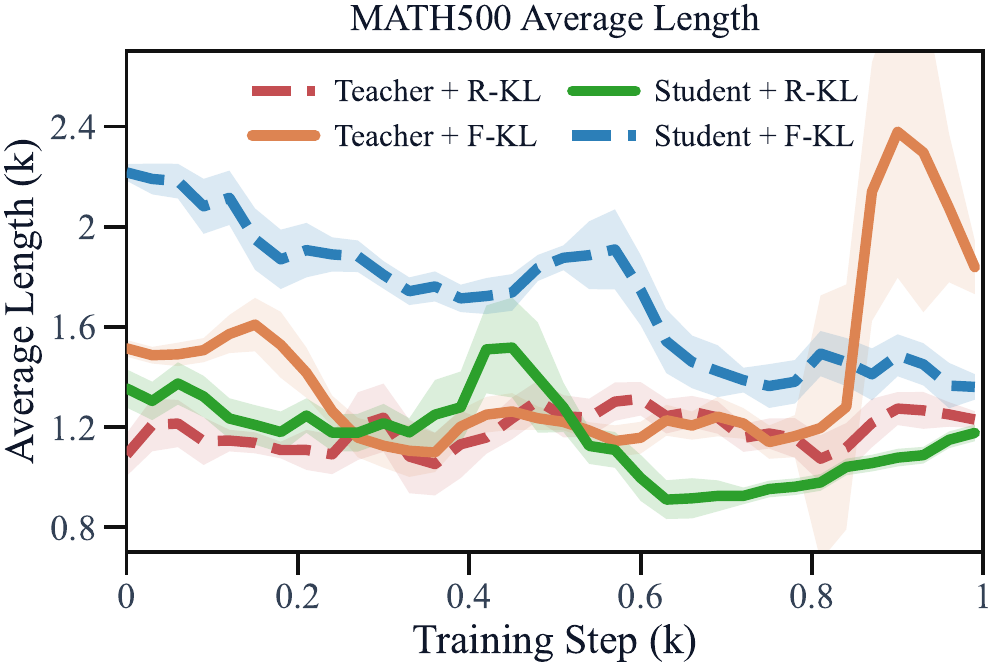}
\end{minipage}
\hfill
\begin{minipage}{0.32\textwidth}
  \centering
  \includegraphics[width=\textwidth]{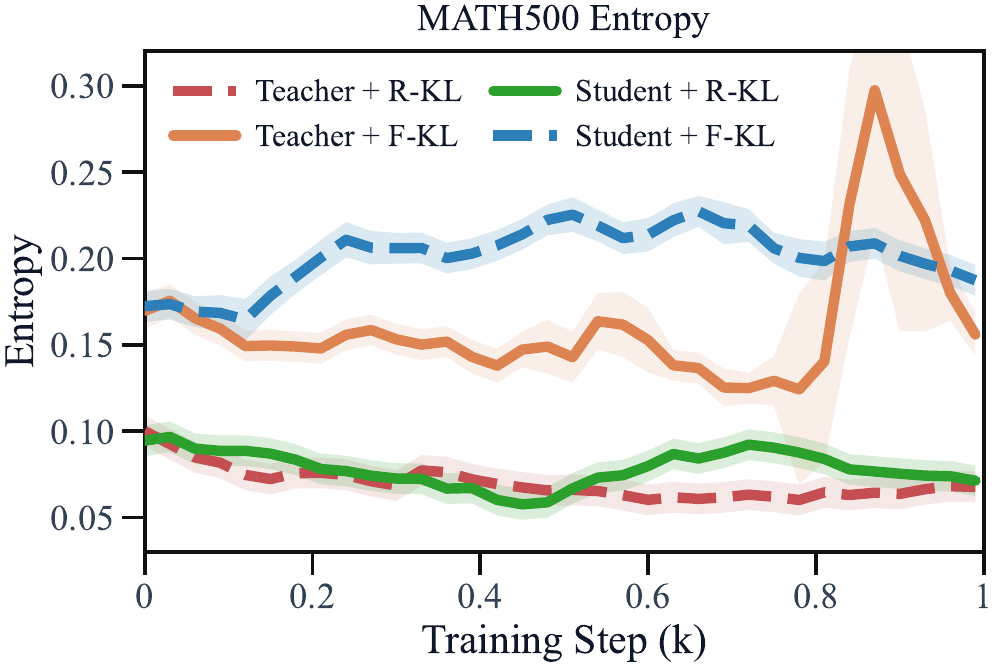}
\end{minipage}


\begin{minipage}{0.32\textwidth}
  \centering
  \includegraphics[width=\textwidth]{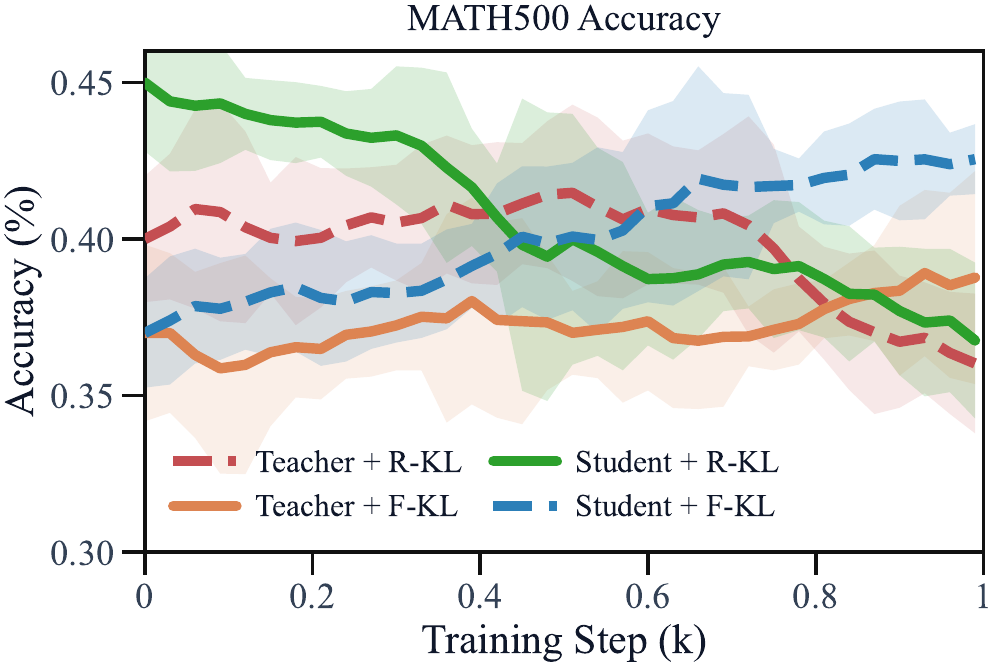}
\end{minipage}
\hfill
\begin{minipage}{0.32\textwidth}
  \centering
  \includegraphics[width=\textwidth]{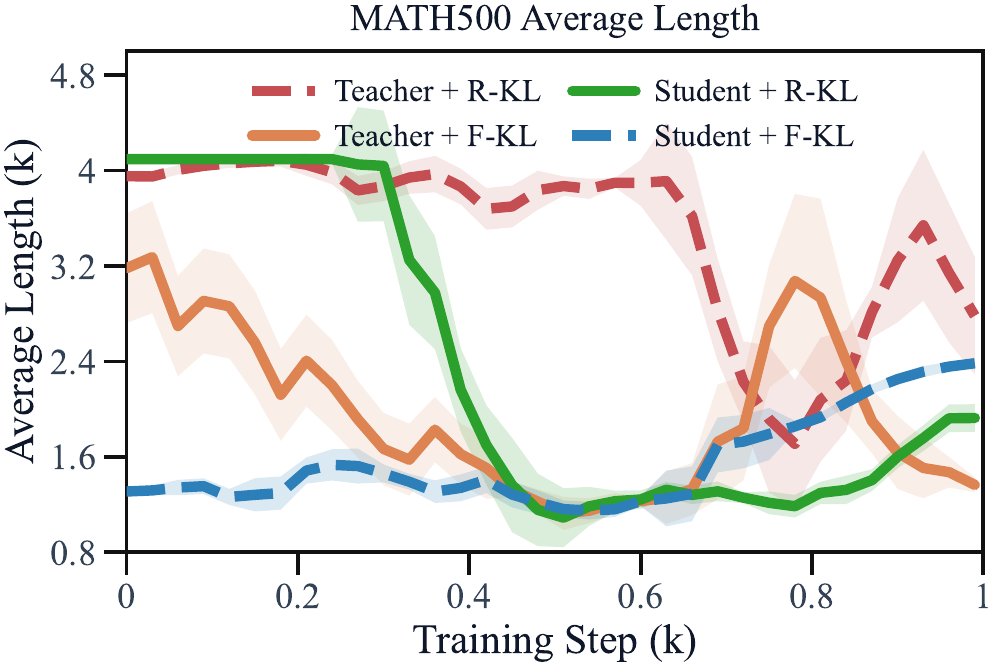}
\end{minipage}
\hfill
\begin{minipage}{0.32\textwidth}
  \centering
  \includegraphics[width=\textwidth]{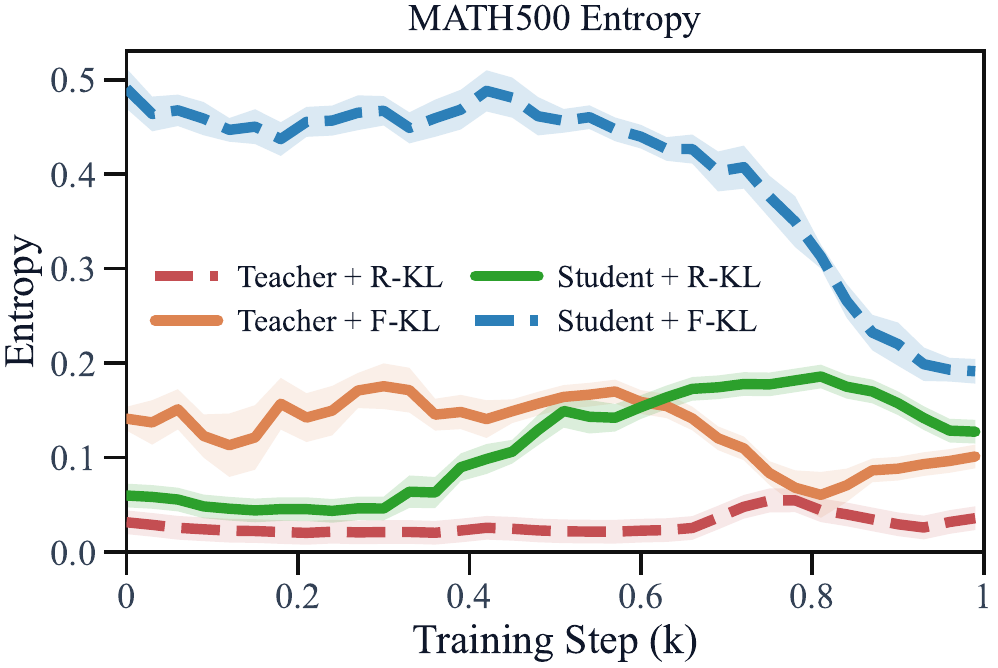}
\end{minipage}

\caption{RL training dynamics after Qwen3-4B-teacher distillation warmup. 
Top/bottom rows: 128/4096-token warmups; left-to-right columns: accuracy, length, entropy.}
\label{fig:math500_grpo}
\end{figure}

\section{Empirical Analysis of the Four Decoupled Objectives}
\label{sec:results}

We analyze the four decoupled objectives along three dimensions: KL direction,
prefix source, and training length. For each dimension, we evaluate the
objectives both as standalone distillation methods and as initializations for
subsequent RL.

\subsection{KL Direction: Forward versus Reverse KL}
\label{sec:results-kl}


\paragraph{Standalone distillation.}
Figures~\ref{fig:math500_dynamics} and~\ref{fig:math500_dynamics_8b} show that
reverse KL consistently outperforms forward KL throughout standalone
distillation across prefix sources, sequence lengths, and teacher scales. This
accuracy advantage comes with a cost: in the 4096-token setting, reverse KL
drives mean per-token predictive entropy close to collapse and often pushes
response lengths toward the evaluation-time generation limit. This matches the
mode-seeking geometry of reverse KL~\citep{gu2023minillm,luo2026demystify},
which concentrates probability mass on a narrower set of high-probability
teacher continuations.
Endpoint evaluations in Tables~\ref{tab:math_eval_128_full} and
\ref{tab:math_eval_4096_full} show the same trend across math benchmarks. Under
matched benchmark, teacher, prefix source, and sequence length, reverse KL
improves Avg@$k$ by $+2.45$ points on average, with larger gains in the
128-token setting ($+3.68$) and smaller but positive gains in the 4096-token
setting ($+1.21$). For example, on MATH500 with student prefixes and 128-token
training, reverse KL raises Avg@$k$ from $34.31\%$ to $42.65\%$ with the
Qwen3-4B teacher and from $34.43\%$ to $43.23\%$ with Qwen3-8B. However, higher
Avg@$k$ does not consistently yield higher Pass@$k$: reverse KL nearly matches
Pass@$k$ in the 128-token setting and falls below forward KL on average in the
4096-token setting. Thus, \textbf{reverse KL achieves stronger Avg@$k$, but the resulting sharper
student distribution reduces diversity, as reflected by matched or lower
Pass@$k$, and destabilizes training under long-sequence distillation.}

\begin{figure}[!t]
\centering
\begin{minipage}{0.32\textwidth}
  \centering
  \includegraphics[width=\textwidth]{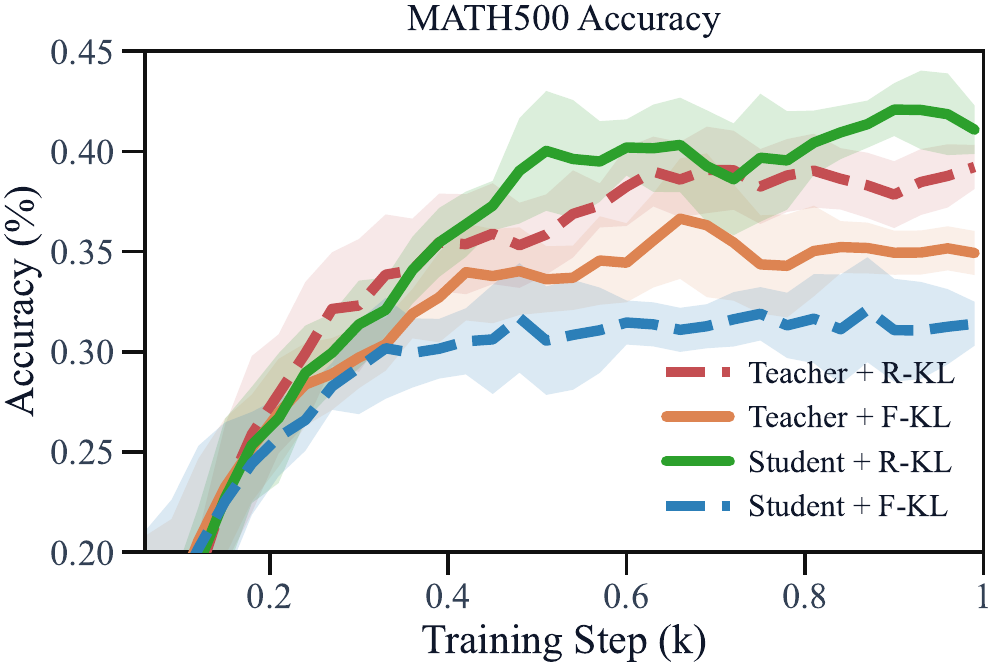}
\end{minipage}
\hfill
\begin{minipage}{0.32\textwidth}
  \centering
  \includegraphics[width=\textwidth]{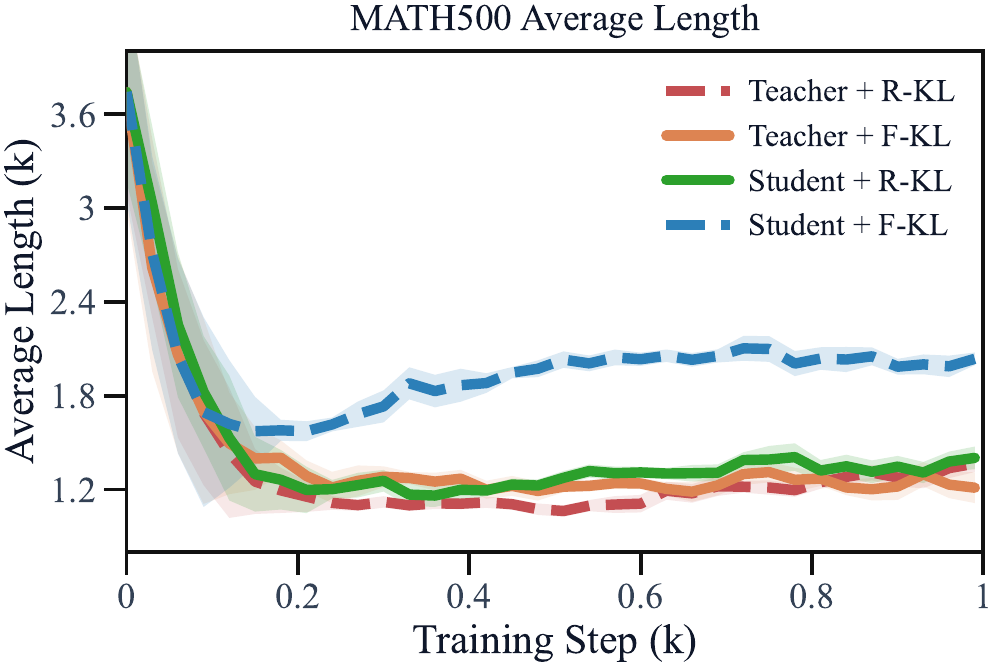}
\end{minipage}
\hfill
\begin{minipage}{0.32\textwidth}
  \centering
  \includegraphics[width=\textwidth]{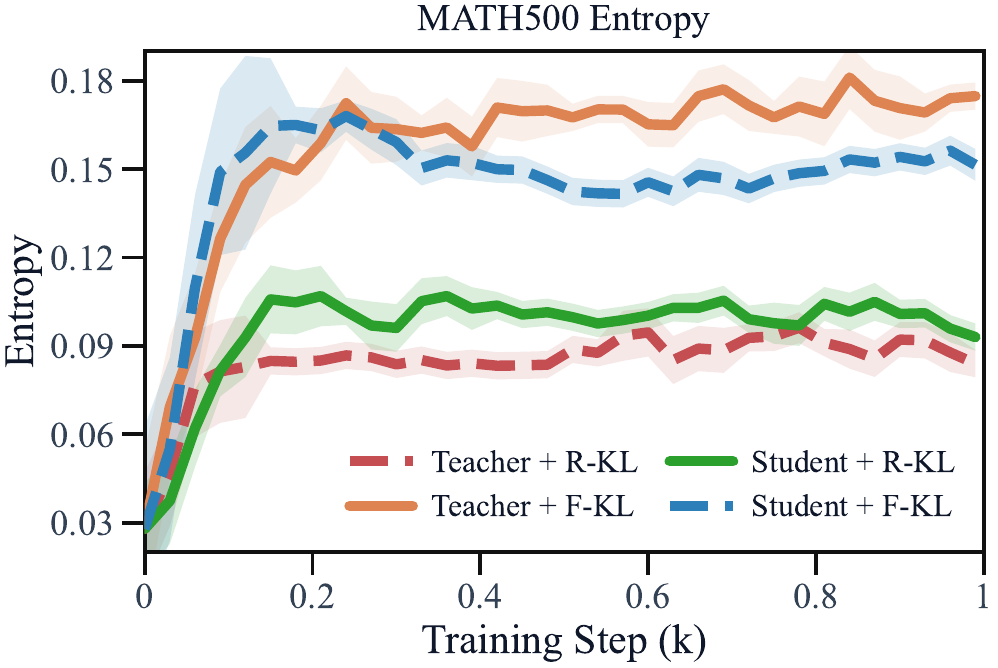}
\end{minipage}


\begin{minipage}{0.32\textwidth}
  \centering
  \includegraphics[width=\textwidth]{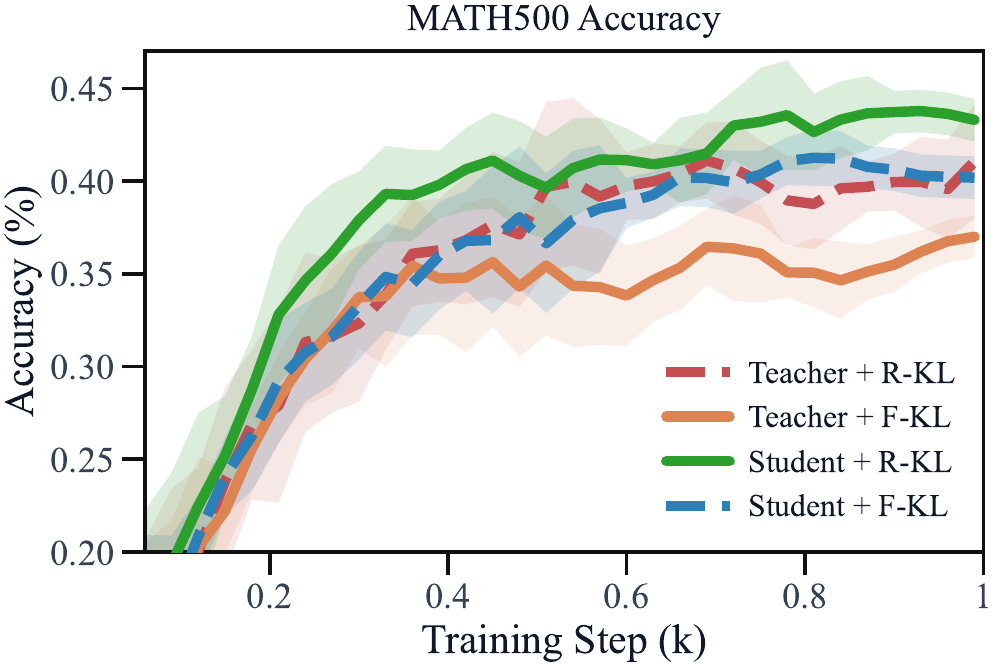}
\end{minipage}
\hfill
\begin{minipage}{0.32\textwidth}
  \centering
  \includegraphics[width=\textwidth]{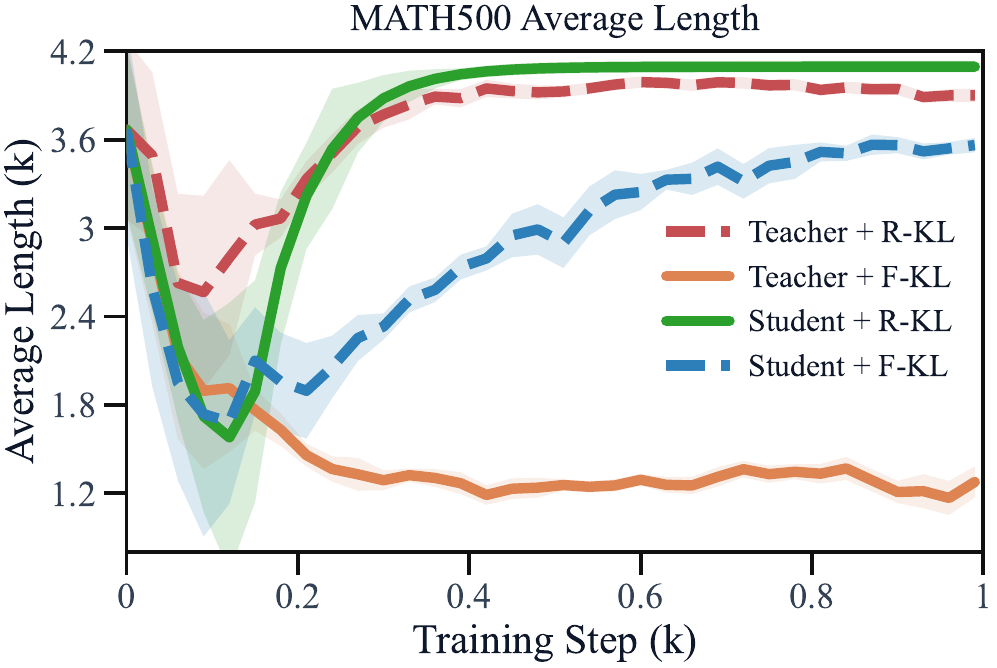}
\end{minipage}
\hfill
\begin{minipage}{0.32\textwidth}
  \centering
  \includegraphics[width=\textwidth]{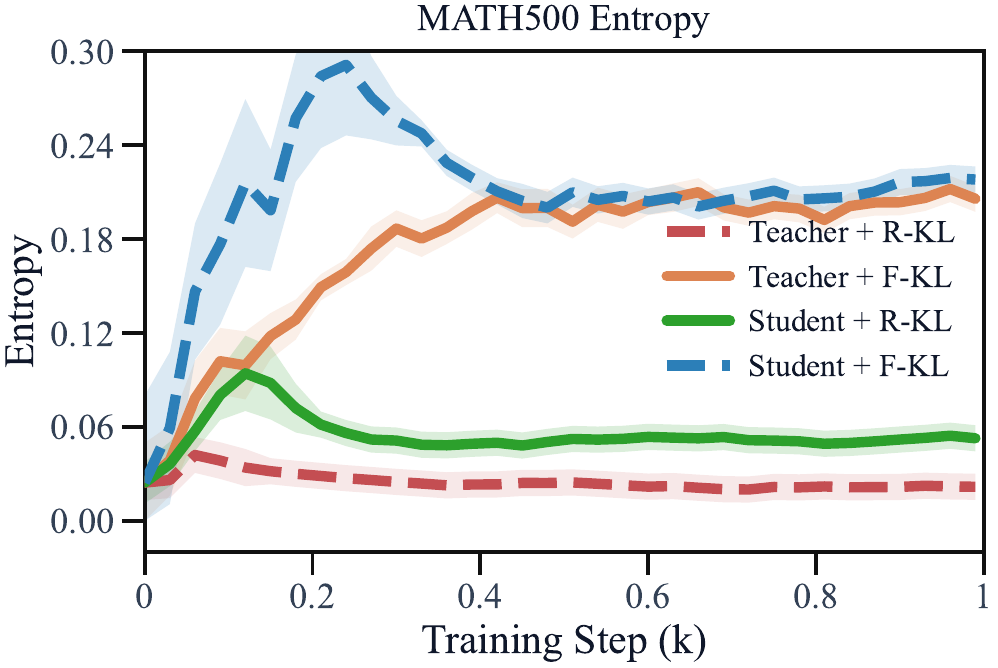}
\end{minipage}

\caption{Distillation training dynamics with Qwen3-8B as teacher and
Qwen3-0.6B as student. Top/bottom rows: 128/4096-token training;
left-to-right columns: accuracy, length, entropy.}
\label{fig:math500_dynamics_8b}
\end{figure}

\begin{figure}[!t]
\centering
\begin{minipage}{0.32\textwidth}
  \centering
  \includegraphics[width=\textwidth]{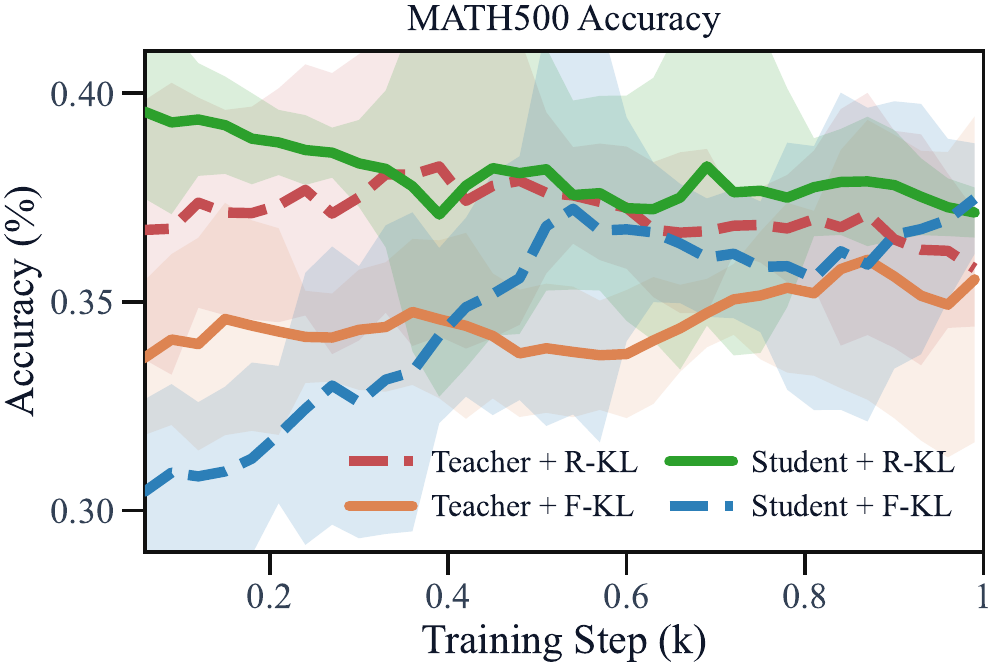}
\end{minipage}
\hfill
\begin{minipage}{0.32\textwidth}
  \centering
  \includegraphics[width=\textwidth]{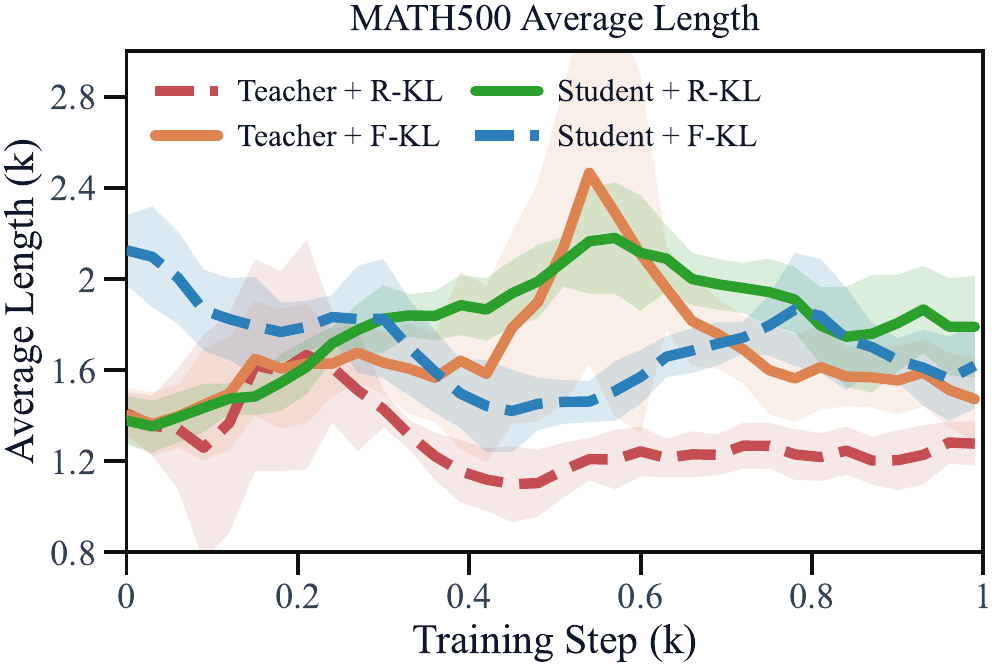}
\end{minipage}
\hfill
\begin{minipage}{0.32\textwidth}
  \centering
  \includegraphics[width=\textwidth]{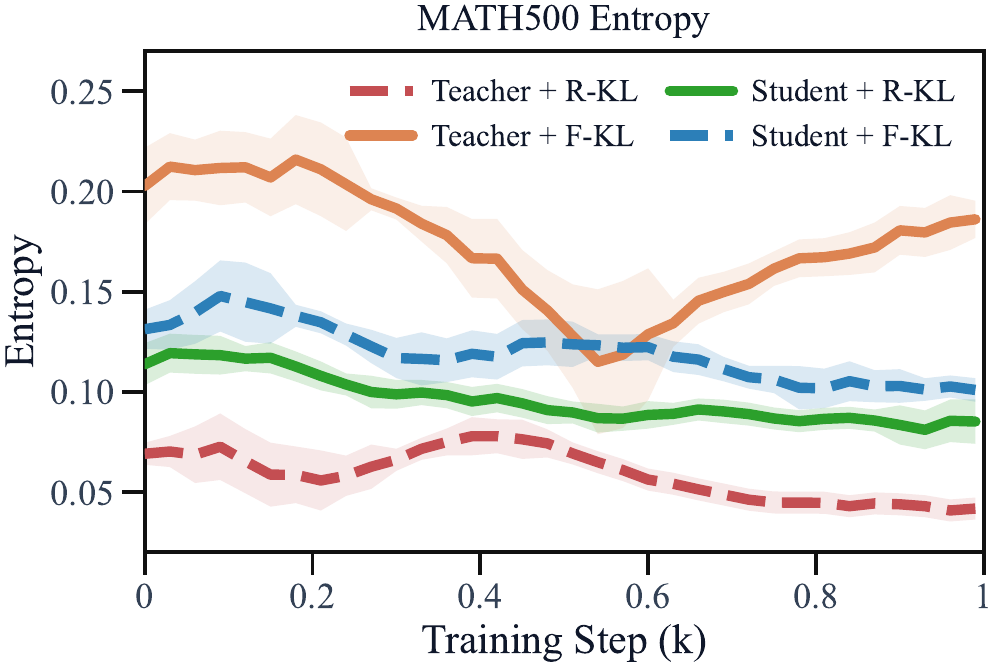}
\end{minipage}


\begin{minipage}{0.32\textwidth}
  \centering
  \includegraphics[width=\textwidth]{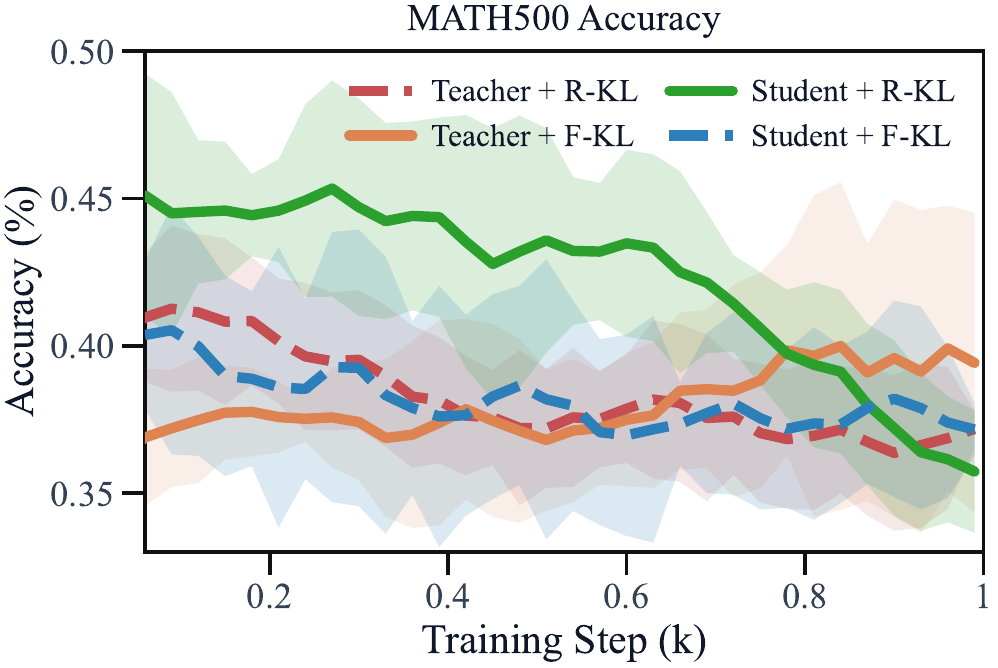}
\end{minipage}
\hfill
\begin{minipage}{0.32\textwidth}
  \centering
  \includegraphics[width=\textwidth]{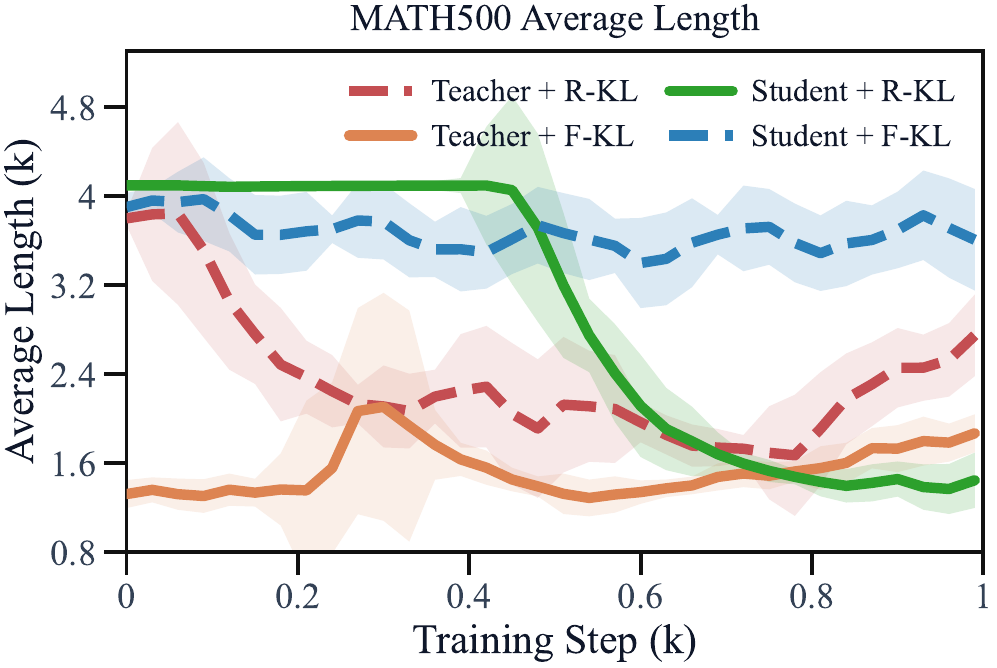}
\end{minipage}
\hfill
\begin{minipage}{0.32\textwidth}
  \centering
  \includegraphics[width=\textwidth]{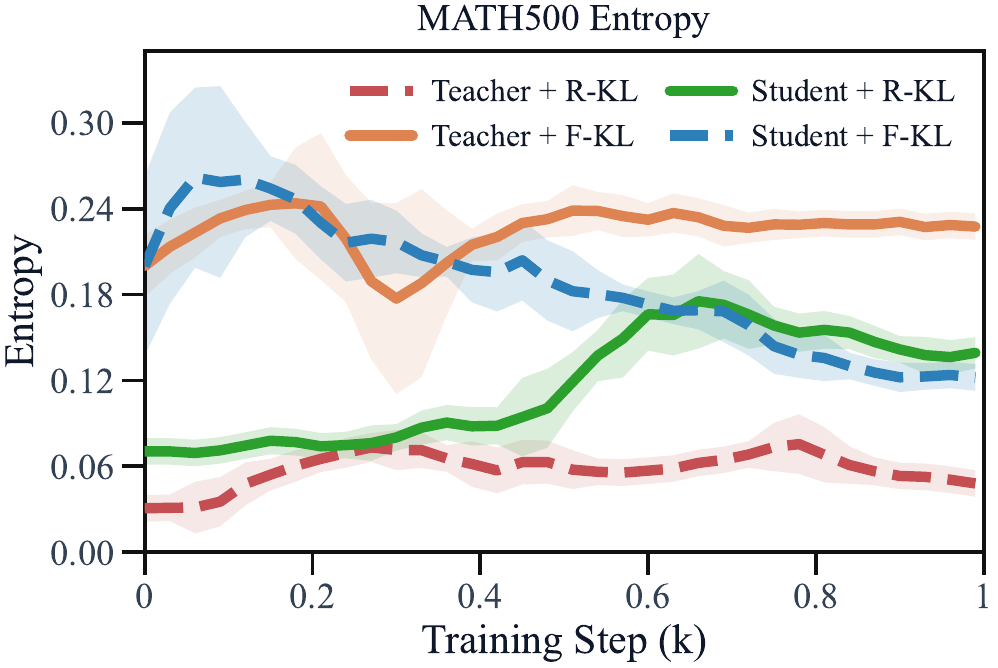}
\end{minipage}

\caption{RL dynamics after Qwen3-8B-teacher warmup. Top/bottom rows:
128/4096-token warmups; left-to-right columns: accuracy, length, entropy.}
\label{fig:math500_grpo_8b}
\end{figure}

\paragraph{RL follow-up.}
The standalone advantage of reverse KL does not reliably transfer to the
subsequent RL stage. Figures~\ref{fig:math500_grpo} and
\ref{fig:math500_grpo_8b} show that reverse-KL warm starts enter GRPO with much
lower predictive entropy than forward-KL warm starts and are more likely to
plateau or degrade. For example, after 4096-token distillation with the
Qwen3-4B teacher, student-prefix reverse KL starts from the strongest checkpoint
at roughly $45\%$ MATH500 accuracy but drops to about $36\%$ during GRPO,
whereas student-prefix forward KL starts lower, around $40\%$, and improves to
about $45\%$. Thus, forward KL closes the initial gap while retaining higher
entropy than reverse KL had at the start of RL. The Qwen3-8B teacher shows the
same pattern in the 128-token setting: student-prefix reverse KL starts higher
but ends near $36\%$, while student-prefix forward KL improves from roughly
$31\%$ to about $36\%$ with consistently higher entropy.
Thus, \textbf{
reverse KL can produce a stronger pre-RL model, but its reduced entropy
constrains exploration and can lead to accuracy degradation, whereas forward KL
is a more reliable initialization for continued policy optimization.}

\begin{figure}[!t]
\centering
\begin{minipage}[t]{0.24\textwidth}
  \vspace{0pt}
  \centering
  \includegraphics[width=\textwidth]{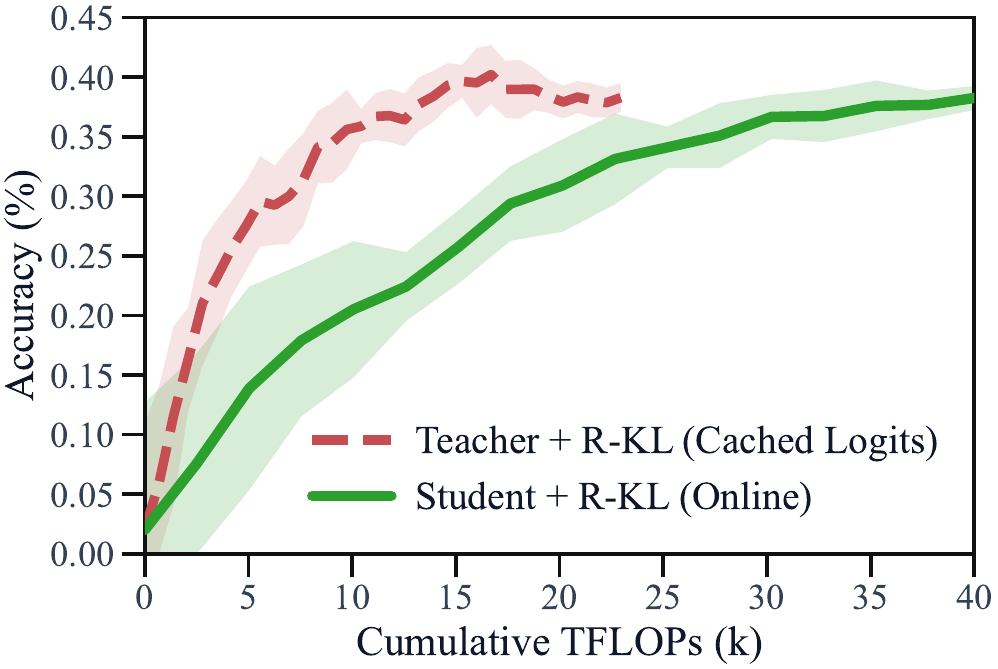}
  {\captionsetup{font=footnotesize}
  \caption{Matched-FLOPs comparison of prefix sources.}
  \label{fig:reverse_flops_128}}
\end{minipage}
\hfill
\begin{minipage}[t]{0.74\textwidth}
  \vspace{0pt}
  \centering
  \begin{minipage}{0.32\textwidth}
    \centering
    \includegraphics[width=\textwidth]{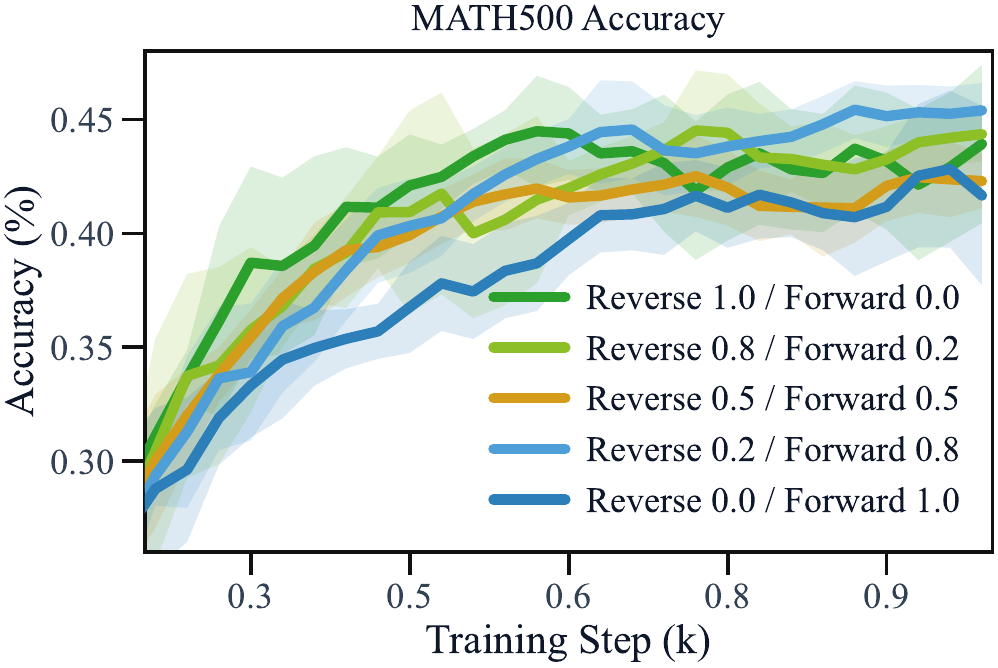}
  \end{minipage}
  \hfill
  \begin{minipage}{0.32\textwidth}
    \centering
    \includegraphics[width=\textwidth]{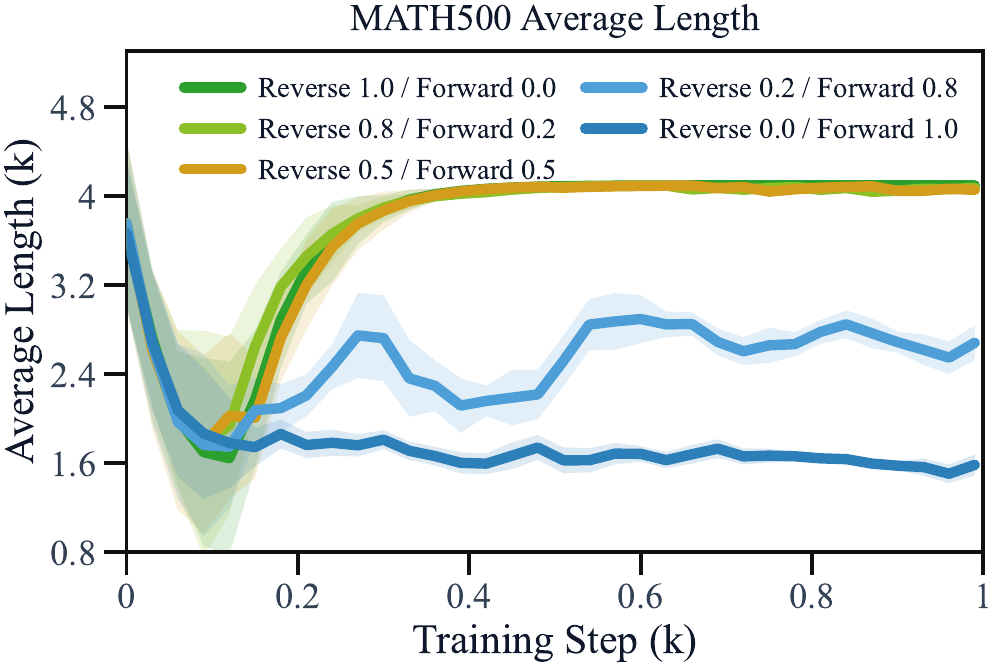}
  \end{minipage}
  \hfill
  \begin{minipage}{0.32\textwidth}
    \centering
    \includegraphics[width=\textwidth]{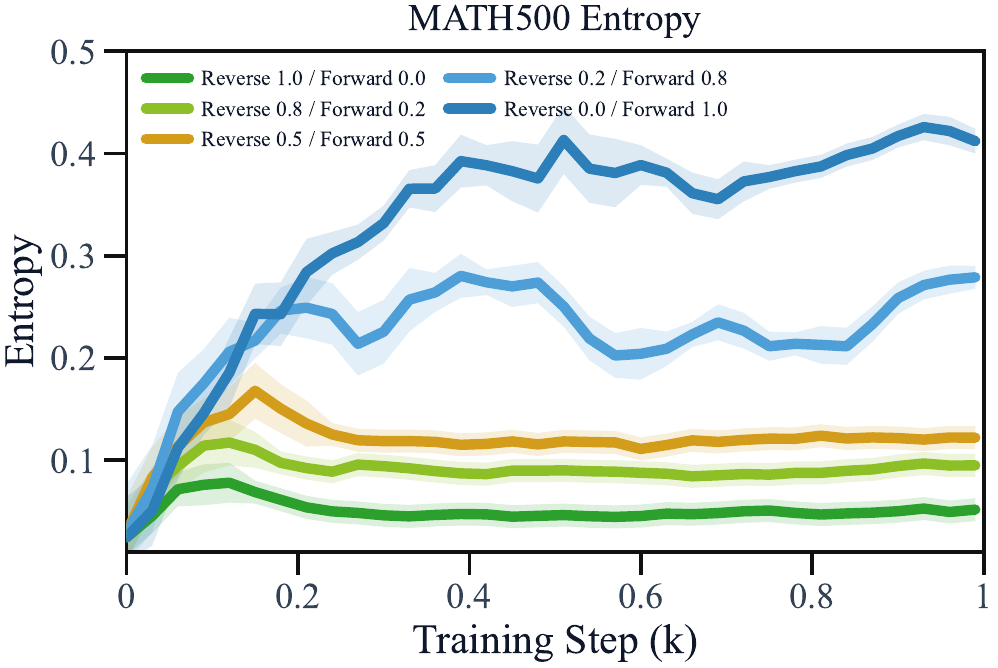}
  \end{minipage}
\caption{Training dynamics of KL mixing on MATH500. Columns, left to right:
accuracy, response length, and entropy.}
  \label{fig:method_forward_plus_reverse}
\end{minipage}
\end{figure}

\subsection{Prefix Source: Teacher versus Student Prefixes}
\label{sec:results-prefix}

\paragraph{Standalone distillation.}
Figures~\ref{fig:math500_dynamics} and~\ref{fig:math500_dynamics_8b}
show stronger accuracy dynamics for student than teacher prefixes, especially
under long-sequence training. This supports the
on-policy motivation: supervision is applied to the states the student actually
visits, rather than only to teacher-induced states. Unlike KL direction, however,
prefix source has a much smaller effect on entropy and length dynamics, which
are largely governed by the KL directions. Endpoint evaluations in
Tables~\ref{tab:math_eval_128_full} and~\ref{tab:math_eval_4096_full} show the
same trend. Across matched comparisons differing only in prefix source, student
prefixes improve Avg@$k$ by $+1.80$ points and Pass@$k$ by $+2.11$ points on
average; after 4096-token distillation, the gains increase to $+3.55$ Avg@$k$
and $+2.95$ Pass@$k$.
Thus, \textbf{Student prefixes improve distillation accuracy, whereas entropy and
length dynamics are largely governed by the KL direction.}

\paragraph{RL follow-up.}
In the RL follow-up, the effect of prefix source remains visible but is less
dominant than the effect of KL direction. Under a fixed KL direction,
student-prefix warm starts are often competitive with or stronger than
teacher-prefix warm starts in accuracy
(Figures~\ref{fig:math500_grpo} and~\ref{fig:math500_grpo_8b}). However, their
entropy and length dynamics are still largely governed by the KL objective:
reverse-KL warm starts remain low-entropy and can plateau or degrade, whereas
forward-KL warm starts preserve higher entropy and are generally more trainable.
Thus, \textbf{student prefixes improve the quality of the distilled
initialization, but KL direction largely determines downstream RL trainability.}

\paragraph{Compute tradeoff.}
Student prefixes improve quality but require online generation. Teacher-prefix
training can instead reuse fixed rollouts and cached teacher logits.
Figure~\ref{fig:reverse_flops_128} shows this tradeoff in the 128-token
reverse-KL setting with the Qwen3-4B teacher: with cached logits, teacher
prefixes reach roughly $38$--$40\%$ MATH500 accuracy within $15$--$20$k
cumulative TFLOPs, while student prefixes require substantially more compute to
reach the same range. Appendix~\ref{app:flops-accounting} details the FLOPs
accounting. \textbf{Thus, student prefixes improve quality under matched training steps,
whereas teacher prefixes can be more compute-efficient under matched FLOPs.}

\subsection{Training Length: Short versus Long Distillation}
\label{sec:results-length}


\paragraph{Standalone distillation.}
Overall, longer distillation improves mathematical reasoning: across
Tables~\ref{tab:math_eval_128_full} and~\ref{tab:math_eval_4096_full}, moving
from 128 to 4096 tokens increases Avg@$k$ by $2.56$ percentage points and
Pass@$k$ by $1.88$ percentage points on average. The gains vary along both
design axes. Along the KL-direction axis, forward KL benefits more, gaining
$3.80$ percentage points in Avg@$k$ versus $1.32$ for reverse KL. Along the
prefix-source axis, student prefixes become more beneficial at longer horizons:
student and teacher prefixes are nearly tied in Avg@$k$ after 128-token
distillation, but after 4096-token distillation student prefixes outperform
teacher prefixes by $3.55$ Avg@$k$ and $2.95$ Pass@$k$ points under matched
teacher scale and KL direction. Reverse KL also benefits from longer sequences,
but its gains come with worse entropy and length dynamics
(Figures~\ref{fig:math500_dynamics} and~\ref{fig:math500_dynamics_8b}):
4096-token reverse-KL runs often drive predictive entropy close to collapse and
induce severe length inflation. Thus, \textbf{longer distillation improves
performance on average. Forward KL benefits more along the KL-direction axis,
student prefixes benefit more along the prefix-source axis, and reverse KL also
gains accuracy but with entropy collapse and severe length inflation.}

\paragraph{RL follow-up.}
When distilled checkpoints initialize RL, increasing distillation length from
128 to 4096 tokens affects forward and reverse KL differently. Figures~\ref{fig:math500_grpo} and~\ref{fig:math500_grpo_8b} show that
longer-sequence warm starts often begin with higher accuracy, but the advantage
is not always preserved during RL.
Forward-KL warm starts generally retain
higher entropy and remain stable or improve, whereas 4096-token reverse-KL warm
starts enter RL with higher initial accuracy but lower entropy than their
128-token counterparts and are more prone to decline. Overall, \textbf{longer
distillation amplifies the entropy gap between KL directions: it makes
forward-KL warm starts more entropic and trainable, but reverse-KL warm starts
lower-entropy despite higher initial accuracy, leading to weaker or declining
RL trajectories.}




\begin{table}[!t]
\caption{Math evaluation after 128-token distillation. Len is the average response length.}
\label{tab:math_eval_128_full}
\centering
\scriptsize
\setlength{\tabcolsep}{2.6pt}
\renewcommand{\arraystretch}{0.7}
\resizebox{\textwidth}{!}{
\begin{tabular}{@{}lllcccccccccccc@{}}
\toprule
\multirow{2}{*}{Prefix} &
\multirow{2}{*}{Teacher} &
\multirow{2}{*}{KL} &
\multicolumn{3}{c}{GSM8K} &
\multicolumn{3}{c}{MATH500} &
\multicolumn{3}{c}{AMC23} &
\multicolumn{3}{c}{AIME24} \\
\cmidrule(lr){4-6}
\cmidrule(lr){7-9}
\cmidrule(lr){10-12}
\cmidrule(l){13-15}
& & & Avg@$k$ & Pass@$k$ & Len & Avg@$k$ & Pass@$k$ & Len & Avg@$k$ & Pass@$k$ & Len & Avg@$k$ & Pass@$k$ & Len \\
\midrule
\multirow{4}{*}{Student}
& \multirow{2}{*}{Qwen3-4B}
& Forward & 62.83 & 77.26 & 1396 & 34.31 & 55.60 & 3876 & 19.00 & 45.00 & 6173 & 0.00 & 0.00 & 7838 \\
& & Reverse & 66.06 & 79.08 & 521 & 42.65 & 55.34 & 2482 & 26.50 & 47.50 & 4464 & 2.67 & 6.67 & 7101 \\
\cmidrule(lr){2-15}
& \multirow{2}{*}{Qwen3-8B}
& Forward & 62.40 & 78.92 & 1279 & 34.43 & 55.98 & 3783 & 19.00 & 50.50 & 6244 & 1.33 & 6.67 & 7697 \\
& & Reverse & 68.03 & 82.11 & 469 & 43.23 & 55.84 & 2364 & 24.00 & 50.00 & 4369 & 2.67 & 6.67 & 6496 \\
\midrule
\multirow{4}{*}{Teacher}
& \multirow{2}{*}{Qwen3-4B}
& Forward & 63.74 & 79.08 & 1012 & 37.72 & 54.14 & 2727 & 20.00 & 45.00 & 4494 & 2.00 & 3.33 & 6607 \\
& & Reverse & 65.13 & 79.45 & 450 & 41.31 & 54.74 & 2140 & 26.00 & 45.00 & 3585 & 0.67 & 3.33 & 6432 \\
\cmidrule(lr){2-15}
& \multirow{2}{*}{Qwen3-8B}
& Forward & 62.72 & 79.30 & 522 & 38.18 & 56.16 & 2171 & 19.50 & 47.50 & 3943 & 2.00 & 6.67 & 6116 \\
& & Reverse & 65.00 & 79.45 & 403 & 43.01 & 56.36 & 2266 & 20.50 & 40.00 & 3800 & 0.67 & 3.33 & 5890 \\
\bottomrule
\end{tabular}
}
\end{table}

\begin{table}[!t]
\caption{Math evaluation after 4096-token distillation. Len is the average response length.}
\label{tab:math_eval_4096_full}
\centering
\scriptsize
\setlength{\tabcolsep}{2.6pt}
\renewcommand{\arraystretch}{0.7}
\resizebox{\textwidth}{!}{
\begin{tabular}{@{}lllcccccccccccc@{}}
\toprule
\multirow{2}{*}{Prefix} &
\multirow{2}{*}{Teacher} &
\multirow{2}{*}{KL} &
\multicolumn{3}{c}{GSM8K} &
\multicolumn{3}{c}{MATH500} &
\multicolumn{3}{c}{AMC23} &
\multicolumn{3}{c}{AIME24} \\
\cmidrule(lr){4-6}
\cmidrule(lr){7-9}
\cmidrule(lr){10-12}
\cmidrule(l){13-15}
& & & Avg@$k$ & Pass@$k$ & Len & Avg@$k$ & Pass@$k$ & Len & Avg@$k$ & Pass@$k$ & Len & Avg@$k$ & Pass@$k$ & Len \\
\midrule
\multirow{4}{*}{Student}
& \multirow{2}{*}{Qwen3-4B}
& Forward & 68.87 & 83.40 & 582 & 46.42 & 60.66 & 2031 & 27.50 & 52.50 & 4109 & 1.33 & 3.33 & 5049 \\
& & Reverse & 69.83 & 82.18 & 8166 & 47.37 & 60.04 & 8189 & 27.00 & 47.50 & 8192 & 1.33 & 3.33 & 8192 \\
\cmidrule(lr){2-15}
& \multirow{2}{*}{Qwen3-8B}
& Forward & 69.04 & 82.79 & 3095 & 45.71 & 60.12 & 6134 & 24.50 & 57.50 & 7138 & 2.67 & 6.67 & 8152 \\
& & Reverse & 69.60 & 81.20 & 8186 & 46.81 & 59.64 & 8191 & 28.00 & 52.50 & 8192 & 2.00 & 3.33 & 8192 \\
\midrule
\multirow{4}{*}{Teacher}
& \multirow{2}{*}{Qwen3-4B}
& Forward & 63.79 & 79.91 & 4807 & 42.09 & 57.08 & 6257 & 20.50 & 42.50 & 7004 & 1.33 & 6.67 & 7575 \\
& & Reverse & 66.64 & 78.92 & 7286 & 43.51 & 56.10 & 7951 & 21.50 & 40.00 & 8079 & 2.00 & 6.67 & 8192 \\
\cmidrule(lr){2-15}
& \multirow{2}{*}{Qwen3-8B}
& Forward & 61.23 & 79.45 & 620 & 41.59 & 58.18 & 1962 & 22.00 & 47.00 & 3447 & 1.33 & 6.67 & 5039 \\
& & Reverse & 66.26 & 78.85 & 6707 & 43.76 & 57.30 & 7669 & 21.00 & 47.50 & 8046 & 2.67 & 6.67 & 8112 \\
\bottomrule
\end{tabular}
}
\end{table}

\begin{table}[t]
\caption{Length curriculum versus fixed 4096-token student-prefix reverse-KL distillation.}
\label{tab:prefix_result_comparison}
\centering
\scriptsize
\setlength{\tabcolsep}{2.6pt}
\resizebox{\textwidth}{!}{
\begin{tabular}{@{}llccccccccccccccc@{}}
\toprule
\multirow{2}{*}{Teacher} &
\multirow{2}{*}{Method} &
\multicolumn{3}{c}{GSM8K} &
\multicolumn{3}{c}{MATH500} &
\multicolumn{3}{c}{AMC23} &
\multicolumn{3}{c}{AIME24} &
\multicolumn{3}{c}{Mean} \\
\cmidrule(lr){3-5}
\cmidrule(lr){6-8}
\cmidrule(lr){9-11}
\cmidrule(lr){12-14}
\cmidrule(l){15-17}
& & Avg@$k$ & Pass@$k$ & Len & Avg@$k$ & Pass@$k$ & Len & Avg@$k$ & Pass@$k$ & Len & Avg@$k$ & Pass@$k$ & Len & Avg@$k$ & Pass@$k$ & Len \\
\midrule
\multirow{2}{*}{Qwen3-4B}
& Fixed length & \textbf{69.8} & \textbf{82.2} & 8166 & 47.4 & 60.0 & 8189 & \textbf{27.0} & 47.5 & 8192 & 1.3 & 3.3 & 8192 & 36.4 & 48.3 & 8185 \\
& Curriculum  & 68.2 & 80.3 & \textbf{455} & \textbf{61.4} & \textbf{73.7} & \textbf{1049} & 26.5 & \textbf{52.5} & \textbf{3695} & \textbf{4.0} & \textbf{10.0} & \textbf{5553} & \textbf{40.0} & \textbf{54.1} & \textbf{2688} \\
\midrule
\multirow{2}{*}{Qwen3-8B}
& Fixed-4096 & 69.6 & 81.2 & 8186 & 46.8 & 59.6 & 8191 & \textbf{28.0} & \textbf{52.5} & 8192 & 2.0 & 3.3 & 8192 & 36.6 & 49.2 & 8190 \\
& Curriculum & \textbf{70.9} & \textbf{83.3} & \textbf{472} & \textbf{60.7} & \textbf{72.0} & \textbf{946} & 26.5 & 45.0 & \textbf{4172} & \textbf{2.7} & \textbf{6.7} & \textbf{5581} & \textbf{40.2} & \textbf{51.8} & \textbf{2793} \\
\bottomrule
\end{tabular}
}
\end{table}

\section{Methods: Balancing Accuracy and Entropy}
\label{sec:methods}

The analysis reveals two tradeoffs. First, the \textbf{KL-direction tradeoff}:
reverse KL achieves stronger Avg@$k$, but entropy collapse reduces diversity,
weakens Pass@$k$, and makes the resulting model less reliable for RL; forward
KL is weaker standalone but better preserves entropy and supports continued RL
improvement. Second, the \textbf{training-length tradeoff}: longer
distillation improves accuracy but amplifies entropy collapse and length
inflation, whereas shorter distillation is more stable but less accurate.
We therefore propose two methods, each targeting one tradeoff.

\subsection{KL Mixing}
\label{sec:kl-mixing}

As shown in Section~\ref{sec:results-kl}, reverse KL provides stronger
standalone distillation but aggressively reduces predictive entropy, whereas
forward KL better preserves entropy but is weaker in Avg@$k$. To address this
KL-direction tradeoff, we propose \textbf{KL mixing}: a distilllation loss that
interpolates between reverse and forward KL at each prefix. For a prefix $s_t$,
we define
\begin{equation}
\label{eq:kl-mixing}
\mathcal{L}_{\lambda}(s_t)
=
\lambda\,
\mathrm{KL}\!\left(q_\theta(\cdot\mid s_t)\,\|\,p_T(\cdot\mid s_t)\right)
+
(1-\lambda)\,
\mathrm{KL}\!\left(p_T(\cdot\mid s_t)\,\|\,q_\theta(\cdot\mid s_t)\right),
\end{equation}
where $\lambda\in[0,1]$ is the reverse-KL mixing weight. We evaluate KL mixing
in the student-prefix, 4096-token setting with Qwen3-4B as teacher. This setting
is where the tradeoff is most pronounced: student-prefix reverse KL achieves the
strongest standalone accuracy, but also exhibits severe entropy collapse and
response-length growth in long-sequence training.

Figure~\ref{fig:method_forward_plus_reverse} shows that KL mixing interpolates
between two imperfect endpoints. Pure reverse KL gives strong accuracy but
quickly lowers entropy and increases response length, whereas pure forward KL
keeps entropy and length stable but gives the weakest accuracy. Intermediate
mixtures trade between these behaviors, but asymmetrically: reverse-heavy and
balanced mixtures raise entropy relative to pure reverse KL, yet still exhibit
length inflation. Surprisingly, the forward-heavy mixture preserves most of the
reverse-KL accuracy, and can even match or slightly exceed it, while increasing
entropy and stabilizing length. These results suggest that \textbf{effective KL mixing
in long-sequence distillation should be forward-heavy}: reverse KL supplies the
transfer signal, but forward KL must carry enough weight to stabilize entropy
and length dynamics.


\subsection{Entropy-Gated Length Curriculum}
\label{sec:length-curriculum}

The second method addresses the training-length tradeoff. Student-prefix reverse
KL achieves the strongest standalone accuracy with long-sequence distillation,
but drives predictive entropy close to zero and pushes response length toward
the generation limit. Short-sequence training is less accurate, but keeps both
entropy and response length stable. This contrast suggests that the best horizon
may lie between the two extremes: long enough to capture much of the
long-sequence accuracy gain, but short enough to avoid entropy collapse and
length inflation.

We therefore propose an \textbf{Entropy-Gated Length Curriculum}. Training starts from a
short training horizon $L_0$ and gradually increases the maximum training
length. During this process, we monitor mean per-token predictive entropy on a
held-out set and track response length as a diagnostic. The curriculum advances
from $L_m$ to $L_{m+1}$ only if the held-out entropy satisfies
$H_m \geq H_{\min}$, where $H_m$ denotes the predictive entropy at training
horizon $L_m$. If this condition fails, we stop increasing the training length
and either terminate distillation or continue training at the last stable
horizon.

We evaluate the length curriculum in the student-prefix reverse-KL setting,
where fixed 4096-token distillation gives strong accuracy but suffers severe
entropy collapse and length inflation. Against fixed 4096-token training with
both Qwen3-4B and Qwen3-8B teachers, Table~\ref{tab:prefix_result_comparison}
shows that the curriculum achieves the intended accuracy--stability tradeoff.
With Qwen3-4B, mean Avg@$k$ improves from $36.4$ to $40.0$ and Pass@$k$ from
$48.3$ to $54.1$, while average length drops from $8185$ to $2688$ tokens; with
Qwen3-8B, Avg@$k$ improves from $36.6$ to $40.2$ and Pass@$k$ from $49.2$ to
$51.8$, while length drops from $8190$ to $2793$. The gains are especially large
on MATH500, where Avg@$k$ increases by roughly $14$ points for both teachers and
length falls from near the generation limit to around $1000$ tokens. Overall,
the curriculum retains the main accuracy benefit of long-horizon distillation
while avoiding length inflation and improving Pass@$k$, which is closely tied
to output diversity.


\section{Conclusion}
\label{sec:conclusion}

We present a decoupled view of LLM distillation that separates prefix source
from KL direction, yielding four objectives connected to off-policy SFT,
DAgger-style on-policy SFT, offline-RL-style distillation, and OPD. Empirically,
these objectives expose accuracy--entropy, quality--compute, and
accuracy--stability tradeoffs across standalone distillation and downstream RL.
We further propose KL mixing and an entropy-gated length curriculum to better
balance distillation performance with diversity and generation stability. Overall,
reasoning distillation should be designed as a controlled tradeoff among
accuracy, diversity, compute, and continued trainability.

\newpage

\bibliographystyle{plainnat}
\bibliography{references}

\clearpage
\appendix
\section*{Appendix}
\section{Related Work}
\label{app:related-work}

\paragraph{Off-policy distillation for LLMs.}
Off-policy distillation Knowledge distillation was originally developed for model compression~\citep{hinton2015kd}.
In modern LLM post-training, Off-policy distillation has increasingly become an important mechanism for
transferring capabilities from strong teachers to trainable students, rather
than merely a tool for compression. For reasoning LLMs, this is often
instantiated as SFT on fixed teacher-generated traces, where sampled teacher
tokens serve as hard labels and provide a single-sample approximation to
forward-KL distillation.
This teacher-trace recipe was made prominent by \citet{deepseek-r1} and has since been adopted and studied in subsequent reasoning-distillation efforts~\citep{muennighoff2025s1,sky_t1_2025,openthoughts,key-factors-cot-distillation}.
These works demonstrate the effectiveness of teacher trajectories, but leave
implicit the coupling between teacher prefixes and forward-KL/SFT-style
supervision.

\paragraph{On-policy distillation for LLMs.}
On-policy distillation was initially proposed to reduce the train--test
distribution mismatch in autoregressive KD~\citep{gu2023minillm,agarwal2023gkd}.
More recently, OPD has emerged as a promising post-training recipe because it
combines two complementary advantages: the on-policy sampling of RL and the
dense token-level supervision of SFT~\citep{lu2025opd}. Motivated by this view,
recent open-source LLM training pipelines have adopted OPD as an important
post-training component~\citep{qwen3,mimo-v2,glm5,nemotron-cascade2,HY-MT1.5,
HY-Embodied-0.5,KAT-Coder-V2,deepseekv4}. In parallel, a growing body of work
has studied the stability of OPD training~\citep{jang2026stableopd,jin2026eopd,
ko2026reopold,fu2026revisitopd,luo2026demystify,li2026rethinkopd,
zhang2026illusion}. Our work studies OPD through the two design choices it
couples by default: student-generated prefixes and reverse-KL token-level
supervision.

\section{Proof of Proposition~\ref{prop:sl-rl}}
\label{app:proof-sl-rl}

\begin{proof}
\noindent\emph{(i) Forward KL.}
For a fixed prefix $s_t$, abbreviate
$p_T(y)=p_T(y\mid s_t)$ and $q_\theta(y)=q_\theta(y\mid s_t)$.
Since $p_T$ is fixed,
\begin{align}
\nabla_\theta \mathrm{KL}(p_T\|q_\theta)(s_t)
&=
\nabla_\theta
\sum_{y\in\mathcal V}
p_T(y)\log\frac{p_T(y)}{q_\theta(y)}
\notag
:=
-\sum_{y\in\mathcal V}
p_T(y)\nabla_\theta\log q_\theta(y)
\notag\\
&=
-\mathbb{E}_{y\sim p_T(\cdot\mid s_t)}
\!\left[\nabla_\theta\log q_\theta(y\mid s_t)\right].
\label{eq:fkl-proof-expansion}
\end{align}
This is the gradient of the cross-entropy
$H(p_T,q_\theta)=-\mathbb{E}_{y\sim p_T(\cdot\mid s_t)}
[\log q_\theta(y\mid s_t)]$, proving Proposition~\ref{prop:sl-rl}(i).

\medskip
\noindent\emph{(ii) Reverse KL.}
For the reverse direction,
\begin{equation}
\mathrm{KL}(q_\theta\|p_T)(s_t)
:=
\sum_{y\in\mathcal V}
q_\theta(y)
\bigl[\log q_\theta(y)-\log p_T(y)\bigr].
\label{eq:rkl-token-objective}
\end{equation}
Taking the gradient gives
\begin{align}
\nabla_\theta \mathrm{KL}(q_\theta\|p_T)(s_t)
&=
\nabla_\theta
\sum_{y\in\mathcal V}
q_\theta(y)\log\frac{q_\theta(y)}{p_T(y)}
\notag
:=
\sum_{y\in\mathcal V}
\bigl[\log q_\theta(y)-\log p_T(y)\bigr]
\nabla_\theta q_\theta(y)
\notag\\
&\quad+
\sum_{y\in\mathcal V}
q_\theta(y)\nabla_\theta\log q_\theta(y)
:=
\sum_{y\in\mathcal V}
\bigl[\log q_\theta(y)-\log p_T(y)+1\bigr]
\nabla_\theta q_\theta(y).
\label{eq:rkl-product-rule}
\end{align}
The $+1$ term vanishes after summing over $y$, since
\begin{equation}
\sum_{y\in\mathcal V}\nabla_\theta q_\theta(y)
:=
\nabla_\theta\sum_{y\in\mathcal V}q_\theta(y)
:=
\nabla_\theta 1
:=
0.
\label{eq:probability-mass-gradient}
\end{equation}
Using the score-function identity,
$\nabla_\theta q_\theta(y)=q_\theta(y)\nabla_\theta\log q_\theta(y)$, we obtain
\begin{align}
\nabla_\theta \mathrm{KL}(q_\theta\|p_T)(s_t)
&=
\sum_{y\in\mathcal V}
q_\theta(y)
\bigl[\log q_\theta(y)-\log p_T(y)\bigr]
\nabla_\theta\log q_\theta(y)
\notag\\
&=
\mathbb{E}_{y\sim q_\theta(\cdot\mid s_t)}
\!\left[
\bigl(\log q_\theta(y\mid s_t)-\log p_T(y\mid s_t)\bigr)
\nabla_\theta\log q_\theta(y\mid s_t)
\right],
\label{eq:rkl-score-form}
\end{align}
which proves Proposition~\ref{prop:sl-rl}(ii). Therefore the negative gradient used for
minimization is
\begin{equation}
-\nabla_\theta \mathrm{KL}(q_\theta\|p_T)(s_t)
:=
\mathbb{E}_{y\sim q_\theta(\cdot\mid s_t)}
\!\left[
\bigl(\log p_T(y\mid s_t)-\log q_\theta(y\mid s_t)\bigr)
\nabla_\theta\log q_\theta(y\mid s_t)
\right],
\label{eq:rkl-negative-gradient}
\end{equation}
the REINFORCE ascent estimator with dense reward
$r(s_t,y)=\log p_T(y\mid s_t)-\log q_\theta(y\mid s_t)$.
\renewcommand{\qedsymbol}{}
\end{proof}

\section{Derivation of Sequence-level KL Decomposition}
\label{app:kl-decomposition}

In this section, we provide the full derivation of the decomposition from sequence-level KL divergence to token-level KL divergence. For simplicity, we consider an output sequence $y=(y_1,\ldots,y_T)$ with fixed length $T$. Variable-length sequences can be handled by treating the end-of-sequence token as part of the vocabulary. Let the state at time $t$ be
\[
s_t = (x, y_{<t}).
\]
Both the teacher distribution $p_T$ and the student distribution $q_\theta$ are autoregressive:
\[
p_T(y|x) = \prod_{t=1}^T p_T(y_t|s_t),
\qquad
q_\theta(y|x) = \prod_{t=1}^T q_\theta(y_t|s_t).
\]
Therefore,
\[
\log p_T(y|x) = \sum_{t=1}^T \log p_T(y_t|s_t),
\qquad
\log q_\theta(y|x) = \sum_{t=1}^T \log q_\theta(y_t|s_t).
\]

\paragraph{Forward KL.}
We first derive the decomposition of the forward KL divergence:
\[
\mathrm{KL}(p_T\|q_\theta)(x)
=
\mathbb{E}_{y\sim p_T(\cdot|x)}
\left[
\log \frac{p_T(y|x)}{q_\theta(y|x)}
\right].
\]
Using the autoregressive factorization, we have
\[
\begin{aligned}
\mathrm{KL}(p_T\|q_\theta)(x)
&=
\mathbb{E}_{y\sim p_T(\cdot|x)}
\left[
\sum_{t=1}^T
\log
\frac{
p_T(y_t|s_t)
}{
q_\theta(y_t|s_t)
}
\right] =
\sum_{t=1}^T
\mathbb{E}_{y\sim p_T(\cdot|x)}
\left[
\log
\frac{
p_T(y_t|s_t)
}{
q_\theta(y_t|s_t)
}
\right].
\end{aligned}
\]
For each time step $t$, we apply the tower property, also known as the law of total expectation, by first conditioning on the prefix state $s_t=(x,y_{<t})$:
\[
\begin{aligned}
&\mathbb{E}_{y\sim p_T(\cdot|x)}
\left[
\log
\frac{
p_T(y_t|s_t)
}{
q_\theta(y_t|s_t)
}
\right] \quad =
\mathbb{E}_{s_t\sim d_T^t}
\left[
\mathbb{E}_{y_t\sim p_T(\cdot|s_t)}
\left[
\log
\frac{
p_T(y_t|s_t)
}{
q_\theta(y_t|s_t)
}
\right]
\right],
\end{aligned}
\]
where $d_T^t$ denotes the distribution over states $s_t=(x,y_{<t})$ induced by rolling out the teacher policy $p_T$ up to time step $t$. The inner expectation is exactly the token-level KL divergence at state $s_t$:
\[
\mathbb{E}_{y_t\sim p_T(\cdot|s_t)}
\left[
\log
\frac{
p_T(y_t|s_t)
}{
q_\theta(y_t|s_t)
}
\right]
=
\mathrm{KL}\left(
p_T(\cdot|s_t)\|q_\theta(\cdot|s_t)
\right).
\]
Thus,
\[
\boxed{
\mathrm{KL}(p_T\|q_\theta)(x)
=
\sum_{t=1}^T
\mathbb{E}_{s_t\sim d_T^t}
\left[
\mathrm{KL}\left(
p_T(\cdot|s_t)\|q_\theta(\cdot|s_t)
\right)
\right].
}
\]
This shows that the sequence-level forward KL decomposes into a sum of token-level forward KL terms, where the states are distributed according to the teacher-induced trajectory distribution.

\paragraph{Reverse KL.}
We can similarly derive the decomposition of the reverse KL divergence:
\[
\mathrm{KL}(q_\theta\|p_T)(x)
=
\mathbb{E}_{y\sim q_\theta(\cdot|x)}
\left[
\log \frac{q_\theta(y|x)}{p_T(y|x)}
\right].
\]
Using the same autoregressive factorization,
\[
\begin{aligned}
\mathrm{KL}(q_\theta\|p_T)(x)
&=
\mathbb{E}_{y\sim q_\theta(\cdot|x)}
\left[
\sum_{t=1}^T
\log
\frac{
q_\theta(y_t|s_t)
}{
p_T(y_t|s_t)
}
\right] =
\sum_{t=1}^T
\mathbb{E}_{y\sim q_\theta(\cdot|x)}
\left[
\log
\frac{
q_\theta(y_t|s_t)
}{
p_T(y_t|s_t)
}
\right].
\end{aligned}
\]
Again applying the tower property at each time step,
\[
\begin{aligned}
&\mathbb{E}_{y\sim q_\theta(\cdot|x)}
\left[
\log
\frac{
q_\theta(y_t|s_t)
}{
p_T(y_t|s_t)
}
\right] =
\mathbb{E}_{s_t\sim d_\theta^t}
\left[
\mathbb{E}_{y_t\sim q_\theta(\cdot|s_t)}
\left[
\log
\frac{
q_\theta(y_t|s_t)
}{
p_T(y_t|s_t)
}
\right]
\right],
\end{aligned}
\]
where $d_\theta^t$ denotes the distribution over states induced by rolling out the student policy $q_\theta$ up to time step $t$. The inner expectation is the token-level reverse KL:
\[
\mathbb{E}_{y_t\sim q_\theta(\cdot|s_t)}
\left[
\log
\frac{
q_\theta(y_t|s_t)
}{
p_T(y_t|s_t)
}
\right]
=
\mathrm{KL}\left(
q_\theta(\cdot|s_t)\|p_T(\cdot|s_t)
\right).
\]
Therefore,
\[
\boxed{
\mathrm{KL}(q_\theta\|p_T)(x)
=
\sum_{t=1}^T
\mathbb{E}_{s_t\sim d_\theta^t}
\left[
\mathrm{KL}\left(
q_\theta(\cdot|s_t)\|p_T(\cdot|s_t)
\right)
\right].
}
\]

The key difference between the two decompositions lies in the state distribution. In the forward KL, token-level KL terms are evaluated under the teacher-induced state distribution $d_T^t$. In the reverse KL, they are evaluated under the student-induced state distribution $d_\theta^t$. Therefore, although both objectives decompose into sums of token-level KL divergences, they couple token-level learning through different trajectory distributions.

\section{Relation between the Reverse-KL Reward and Schulman's KL Estimators}
\label{app:schulman-kl-estimators}

In this section, we clarify the connection between our reverse-KL reward and the KL estimators discussed by Schulman. Consider a fixed state $s_t=(x,y_{<t})$. Let $q_\theta(\cdot|s_t)$ denote the student policy and $p_T(\cdot|s_t)$ denote the teacher policy. The token-level reverse KL is
\[
\mathrm{KL}\left(q_\theta(\cdot|s_t)\|p_T(\cdot|s_t)\right)
=
\mathbb{E}_{y\sim q_\theta(\cdot|s_t)}
\left[
\log \frac{q_\theta(y|s_t)}{p_T(y|s_t)}
\right].
\]
Define the likelihood ratio
\[
\rho(y,s_t)
=
\frac{p_T(y|s_t)}{q_\theta(y|s_t)}.
\]
Then the reverse KL can be written as
\[
\mathrm{KL}\left(q_\theta(\cdot|s_t)\|p_T(\cdot|s_t)\right)
=
\mathbb{E}_{y\sim q_\theta(\cdot|s_t)}
\left[
-\log \rho(y,s_t)
\right].
\]

\paragraph{Connection to the log-ratio reward.}
The reward used in Proposition~\autoref{prop:sl-rl}(ii) is
\[
r_{\mathrm{KL}}(y,s_t)
=
\log p_T(y|s_t)-\log q_\theta(y|s_t)
=
\log \rho(y,s_t).
\]
Therefore,
\[
\mathbb{E}_{y\sim q_\theta(\cdot|s_t)}
\left[
r_{\mathrm{KL}}(y,s_t)
\right]
=
-\mathrm{KL}\left(q_\theta(\cdot|s_t)\|p_T(\cdot|s_t)\right).
\]
Thus, maximizing the expected reward $r_{\mathrm{KL}}$ is equivalent to minimizing the token-level reverse KL. Equivalently, if the KL is written as a penalty, the corresponding single-sample estimator is
\[
k_1(y,s_t)
=
-r_{\mathrm{KL}}(y,s_t)
=
\log q_\theta(y|s_t)-\log p_T(y|s_t).
\]
This is Schulman's $k_1$ estimator for the reverse KL under samples from $q_\theta$. It is unbiased:
\[
\mathbb{E}_{y\sim q_\theta(\cdot|s_t)}
\left[
k_1(y,s_t)
\right]
=
\mathrm{KL}\left(q_\theta(\cdot|s_t)\|p_T(\cdot|s_t)\right),
\]
but it is not guaranteed to be non-negative for each individual sample.

\paragraph{Schulman's three KL estimators.}
Using the ratio
\[
\rho(y,s_t)
=
\frac{p_T(y|s_t)}{q_\theta(y|s_t)},
\]
the three commonly used single-sample KL estimators for
$\mathrm{KL}(q_\theta\|p_T)$ are
\[
k_1(y,s_t)
=
-\log \rho(y,s_t),
\]
\[
k_2(y,s_t)
=
\frac{1}{2}
\left(
\log \rho(y,s_t)
\right)^2,
\]
and
\[
k_3(y,s_t)
=
\rho(y,s_t)-1-\log \rho(y,s_t).
\]

The $k_1$ estimator is unbiased because
\[
\mathbb{E}_{y\sim q_\theta}
\left[
-\log \rho(y,s_t)
\right]
=
\mathbb{E}_{y\sim q_\theta}
\left[
\log \frac{q_\theta(y|s_t)}{p_T(y|s_t)}
\right]
=
\mathrm{KL}\left(q_\theta(\cdot|s_t)\|p_T(\cdot|s_t)\right).
\]
However, $k_1$ can be negative for individual samples, even though its expectation is non-negative.

The $k_2$ estimator,
\[
k_2(y,s_t)
=
\frac{1}{2}
\left(
\log \rho(y,s_t)
\right)^2,
\]
is always non-negative. It is not an unbiased estimator of the KL divergence, but it is a second-order approximation when $p_T$ and $q_\theta$ are close. To see this, let
\[
\delta(y,s_t)=\log \rho(y,s_t).
\]
Since
\[
\mathbb{E}_{y\sim q_\theta}
\left[
\rho(y,s_t)
\right]
=
\sum_y q_\theta(y|s_t)
\frac{p_T(y|s_t)}{q_\theta(y|s_t)}
=
1,
\]
we have
\[
\mathbb{E}_{y\sim q_\theta}
\left[
e^{\delta(y,s_t)}
\right]
=
1.
\]
Using the Taylor expansion $e^\delta=1+\delta+\frac{1}{2}\delta^2+O(\delta^3)$ gives
\[
0
=
\mathbb{E}_{y\sim q_\theta}
\left[
\delta(y,s_t)
+
\frac{1}{2}\delta(y,s_t)^2
+
O(\delta(y,s_t)^3)
\right].
\]
Therefore,
\[
-\mathbb{E}_{y\sim q_\theta}
\left[
\delta(y,s_t)
\right]
=
\frac{1}{2}
\mathbb{E}_{y\sim q_\theta}
\left[
\delta(y,s_t)^2
\right]
+
O(\delta^3).
\]
Since the left-hand side is exactly
\[
\mathrm{KL}\left(q_\theta(\cdot|s_t)\|p_T(\cdot|s_t)\right),
\]
we obtain the local approximation
\[
\mathrm{KL}\left(q_\theta(\cdot|s_t)\|p_T(\cdot|s_t)\right)
\approx
\mathbb{E}_{y\sim q_\theta}
\left[
\frac{1}{2}
\left(
\log \rho(y,s_t)
\right)^2
\right].
\]

The $k_3$ estimator is
\[
k_3(y,s_t)
=
\rho(y,s_t)-1-\log \rho(y,s_t).
\]
It is unbiased because
\[
\begin{aligned}
\mathbb{E}_{y\sim q_\theta}
\left[
k_3(y,s_t)
\right]
&=
\mathbb{E}_{y\sim q_\theta}
\left[
\rho(y,s_t)-1-\log \rho(y,s_t)
\right] \\
&=
\mathbb{E}_{y\sim q_\theta}
\left[
\rho(y,s_t)
\right]
-
1
-
\mathbb{E}_{y\sim q_\theta}
\left[
\log \rho(y,s_t)
\right] \\
&=
1-1
+
\mathrm{KL}\left(q_\theta(\cdot|s_t)\|p_T(\cdot|s_t)\right) \\
&=
\mathrm{KL}\left(q_\theta(\cdot|s_t)\|p_T(\cdot|s_t)\right).
\end{aligned}
\]
Moreover, $k_3$ is non-negative for every sample because
\[
z-1-\log z \ge 0
\qquad
\text{for all } z>0,
\]
with equality if and only if $z=1$. Taking $z=\rho(y,s_t)$ gives
\[
k_3(y,s_t)\ge 0.
\]

\paragraph{Implementation used in this work.}
Our implementation uses the log-ratio reward form
\[
r_{\mathrm{KL}}(y,s_t)
=
\log p_T(y|s_t)-\log q_\theta(y|s_t).
\]
Equivalently, when written as a KL penalty, this corresponds to the $k_1$ estimator
\[
k_1(y,s_t)
=
\log q_\theta(y|s_t)-\log p_T(y|s_t).
\]
Thus, our method directly optimizes the negative single-sample reverse-KL estimator as a reward. We do not use the squared-log approximation $k_2$ or the non-negative unbiased estimator $k_3$ in the reward. In terms of the reward notation $r_{\mathrm{KL}}=\log p_T-\log q_\theta$, the three corresponding penalty forms are
\[
k_1=-r_{\mathrm{KL}},
\qquad
k_2=\frac{1}{2}r_{\mathrm{KL}}^2,
\qquad
k_3=\exp(r_{\mathrm{KL}})-1-r_{\mathrm{KL}}.
\]
This distinction is important: the reward in Proposition~\autoref{prop:sl-rl}(ii) is a log-ratio reward, i.e., the negative of the $k_1$ reverse-KL penalty estimator, rather than the $k_2$ or $k_3$ KL penalty.

\section{Connection between Student-Prefix Forward KL and DAgger}
\label{app:dagger-connection}

In this section, we clarify the connection between the student-prefix forward-KL objective and DAgger-style imitation learning. The main point is that student-prefix forward KL corresponds to the supervised learning subproblem of DAgger when the expert is a soft teacher distribution and the imitation loss is cross-entropy.

\paragraph{DAgger-style imitation learning.}
DAgger \cite{ross2011dagger} is an iterative imitation-learning algorithm designed to address distribution shift caused by training only on expert-induced states. At iteration $k$, the learner policy $\pi_k$ is rolled out to collect states from the learner-induced state distribution. The expert policy $\pi^\star$ is then queried on these states to provide supervision. The collected state-action pairs are aggregated into a dataset, and the next learner policy is trained by minimizing a supervised imitation loss on the aggregated dataset.

In the standard hard-label setting, the expert provides an action
\[
a^\star \sim \pi^\star(\cdot|s),
\]
and the learner is trained to minimize a loss of the form
\[
\ell(\pi_\theta(s), a^\star).
\]
In our setting, the teacher provides a full distribution over next tokens, so the expert is a soft-label expert rather than a hard-label expert.

\paragraph{Correspondence to our setting.}
For autoregressive generation, the state at time $t$ is
\[
s_t = (x, y_{<t}).
\]
The student policy is
\[
\pi_\theta(\cdot|s_t) = q_\theta(\cdot|s_t),
\]
and the expert policy is the teacher distribution
\[
\pi^\star(\cdot|s_t) = p_T(\cdot|s_t).
\]
Let $d_\theta^t$ denote the distribution over prefixes $s_t$ induced by rolling out the student policy $q_\theta$ up to time step $t$.

The student-prefix forward-KL objective is
\[
\mathcal{L}_{\mathrm{SP\text{-}FKL}}(\theta)
=
\sum_{t=1}^T
\mathbb{E}_{s_t\sim d_\theta^t}
\left[
\mathrm{KL}
\left(
p_T(\cdot|s_t)\|q_\theta(\cdot|s_t)
\right)
\right].
\]
For a fixed state $s_t$, the forward KL can be written as
\[
\begin{aligned}
\mathrm{KL}
\left(
p_T(\cdot|s_t)\|q_\theta(\cdot|s_t)
\right)
&=
\sum_y
p_T(y|s_t)
\log
\frac{
p_T(y|s_t)
}{
q_\theta(y|s_t)
} \\
&=
\sum_y
p_T(y|s_t)
\log p_T(y|s_t)
-
\sum_y
p_T(y|s_t)
\log q_\theta(y|s_t) \\
&=
-H\left(p_T(\cdot|s_t)\right)
+
H\left(p_T(\cdot|s_t), q_\theta(\cdot|s_t)\right),
\end{aligned}
\]
where
\[
H\left(p_T(\cdot|s_t), q_\theta(\cdot|s_t)\right)
=
-\sum_y p_T(y|s_t)\log q_\theta(y|s_t)
\]
is the cross-entropy from the teacher distribution to the student distribution. Since
\[
H\left(p_T(\cdot|s_t)\right)
=
-\sum_y p_T(y|s_t)\log p_T(y|s_t)
\]
does not depend on the student parameters $\theta$, minimizing
\[
\mathrm{KL}
\left(
p_T(\cdot|s_t)\|q_\theta(\cdot|s_t)
\right)
\]
with respect to $\theta$ is equivalent to minimizing the soft-label cross-entropy
\[
H\left(p_T(\cdot|s_t), q_\theta(\cdot|s_t)\right).
\]
Therefore,
\[
\arg\min_\theta
\mathcal{L}_{\mathrm{SP\text{-}FKL}}(\theta)
=
\arg\min_\theta
\sum_{t=1}^T
\mathbb{E}_{s_t\sim d_\theta^t}
\left[
H\left(
p_T(\cdot|s_t),
q_\theta(\cdot|s_t)
\right)
\right],
\]
up to terms independent of $\theta$.

This is precisely the DAgger-style supervised imitation objective under the following correspondence:
\[
\text{DAgger expert } \pi^\star
\quad \leftrightarrow \quad
\text{teacher } p_T,
\]
\[
\text{DAgger learner } \pi_\theta
\quad \leftrightarrow \quad
\text{student } q_\theta,
\]
\[
\text{DAgger state distribution } d_{\pi_\theta}
\quad \leftrightarrow \quad
\text{student-prefix distribution } d_\theta,
\]
\[
\text{DAgger supervised loss}
\quad \leftrightarrow \quad
\text{soft-label cross-entropy / forward KL}.
\]

\paragraph{Equivalence to the DAgger supervised subproblem.}
At iteration $k$, suppose the student policy $q_{\theta_k}$ is rolled out to collect prefixes
\[
s_t \sim d_{\theta_k}^t.
\]
The teacher is then queried at each collected prefix to obtain the soft target distribution
\[
p_T(\cdot|s_t).
\]
The supervised update solves
\[
\theta_{k+1}
\in
\arg\min_\theta
\sum_{i=1}^k
\sum_{t=1}^T
\mathbb{E}_{s_t\sim d_{\theta_i}^t}
\left[
H\left(
p_T(\cdot|s_t),
q_\theta(\cdot|s_t)
\right)
\right],
\]
where the sum over $i$ corresponds to dataset aggregation across previous student rollouts. Equivalently, since the teacher entropy term is independent of $\theta$, this can be written as
\[
\theta_{k+1}
\in
\arg\min_\theta
\sum_{i=1}^k
\sum_{t=1}^T
\mathbb{E}_{s_t\sim d_{\theta_i}^t}
\left[
\mathrm{KL}
\left(
p_T(\cdot|s_t)\|q_\theta(\cdot|s_t)
\right)
\right].
\]
Thus, with dataset aggregation, student-prefix forward KL exactly matches the supervised learning subproblem solved by DAgger, generalized from hard expert actions to soft teacher distributions.

If the implementation does not aggregate all previous student-prefix data and instead trains only on the current student-induced distribution $d_{\theta_k}$, then the method should be understood as an on-policy or DAgger-style variant rather than a literal implementation of the original DAgger algorithm.

\paragraph{Relation to hard-label DAgger.}
The connection becomes especially clear when the teacher distribution is deterministic. Suppose
\[
p_T(y|s_t) = \mathbf{1}\{y = y_T^\star(s_t)\}.
\]
Then the soft-label cross-entropy reduces to
\[
H\left(p_T(\cdot|s_t), q_\theta(\cdot|s_t)\right)
=
-\log q_\theta(y_T^\star(s_t)|s_t).
\]
This is the standard negative log-likelihood imitation loss on the expert action. Therefore, hard-label DAgger can be viewed as a special case of the student-prefix forward-KL objective where the teacher distribution is a point mass.

\paragraph{Difference from the original DAgger guarantee.}
The original DAgger analysis is typically stated in terms of online no-regret learning with imitation losses such as classification or surrogate losses, and it is often connected to bounds on task performance under the learner-induced state distribution. Our objective uses a KL or cross-entropy imitation loss with a soft teacher distribution. Therefore, the correspondence above should be interpreted as an objective-level equivalence to the DAgger supervised learning subproblem, not as a direct reuse of the original DAgger performance bound.

In summary, student-prefix forward KL can be viewed as a soft-label DAgger-style objective: it trains the student to match the teacher on states induced by the student's own rollouts. The essential DAgger principle is preserved because supervision is applied under the learner-induced state distribution rather than only under the teacher-induced state distribution.

\section{Offline-RL Interpretation of Teacher-Prefix Reverse KL}
\label{app:offline-rl-interpretation}

We briefly clarify the connection between teacher-prefix reverse KL and offline policy optimization. In standard offline RL, a behavior policy collects a fixed dataset, and a target policy is optimized using this fixed distribution without further environment interaction. In our setting, the teacher plays the role of the behavior policy: prefixes are sampled from the teacher-induced distribution $d_T^t$, while the student $q_\theta$ is the target policy optimized on these prefixes.

The teacher-prefix reverse-KL objective is
\[
\mathcal{L}_{\mathrm{TP\text{-}RKL}}(\theta)
=
\sum_{t=1}^T
\mathbb{E}_{s_t\sim d_T^t}
\left[
\mathrm{KL}\left(
q_\theta(\cdot|s_t)\|p_T(\cdot|s_t)
\right)
\right].
\]
Expanding the KL term gives
\[
\begin{aligned}
\mathcal{L}_{\mathrm{TP\text{-}RKL}}(\theta)
&=
\sum_{t=1}^T
\mathbb{E}_{s_t\sim d_T^t}
\mathbb{E}_{y_t\sim q_\theta(\cdot|s_t)}
\left[
\log q_\theta(y_t|s_t)-\log p_T(y_t|s_t)
\right].
\end{aligned}
\]
Therefore, minimizing teacher-prefix reverse KL is equivalent to maximizing
\[
\mathcal{J}_{\mathrm{TP\text{-}RKL}}(\theta)
=
\sum_{t=1}^T
\mathbb{E}_{s_t\sim d_T^t}
\mathbb{E}_{y_t\sim q_\theta(\cdot|s_t)}
\left[
\log p_T(y_t|s_t)-\log q_\theta(y_t|s_t)
\right].
\]
This has the form of an entropy-regularized policy optimization objective on an offline state distribution. The reward term is $\log p_T(y_t|s_t)$, and the entropy regularization is given by $-\log q_\theta(y_t|s_t)$. Equivalently, one can view the per-token log-ratio reward as
\[
r_{\mathrm{KL}}(y_t,s_t)
=
\log p_T(y_t|s_t)-\log q_\theta(y_t|s_t).
\]

Since the state distribution $d_T^t$ is fixed with respect to $\theta$, the gradient takes a policy-gradient form:
\[
\begin{aligned}
\nabla_\theta \mathcal{J}_{\mathrm{TP\text{-}RKL}}(\theta)
&=
\sum_{t=1}^T
\mathbb{E}_{s_t\sim d_T^t}
\mathbb{E}_{y_t\sim q_\theta(\cdot|s_t)}
\left[
\nabla_\theta \log q_\theta(y_t|s_t)
\left(
\log p_T(y_t|s_t)-\log q_\theta(y_t|s_t)
\right)
\right].
\end{aligned}
\]
Thus, the teacher-prefix reverse-KL update can be interpreted as an offline policy-gradient update: states are drawn from the teacher-induced distribution, actions are sampled from the current student policy, and the reward is the log-ratio reward above. No action-level importance sampling is required because the action expectation is already taken under $q_\theta(\cdot|s_t)$ rather than under the teacher.

This interpretation is only an objective-level analogy to offline RL. Unlike standard offline RL methods such as CQL \citep{CQL} or IQL \citep{IQL}, this objective does not learn a value function, does not perform Bellman backups, and does not optimize the closed-loop trajectory distribution induced by the student. It instead optimizes the student policy on teacher-induced prefixes, with the teacher likelihood providing a dense per-token reward. Therefore, teacher-prefix reverse KL should be viewed as an offline-RL-style policy optimization objective, rather than as an instance of standard offline RL algorithms.

\section{Detailed Experimental Settings}
\label{app:detailed-setup}

\subsection{Model Family}
\label{app:models-scope}

We use Qwen3-0.6B-Base as the student model, with Qwen3-4B-Base and Qwen3-8B-Base serving as teacher models. These models are from the same Qwen3 family and share the same tokenizer and architecture, enabling a controlled comparison across model scales.

\subsection{Training Data and Prefix Construction}
\label{app:data-prefixes}

The core training domain is mathematical reasoning. For standalone distillation, all objectives use the same prompt distribution drawn from DeepScaleR, so that the problem distribution is held fixed while only the prefix source and KL direction vary. For teacher-prefix objectives, we generate offline teacher rollouts from this shared prompt set, compute the corresponding logits, and reuse the resulting prefixes and cached logits throughout training. In practice, these rollouts are generated with temperature~1.0, top-$p$~0.95, top-$k$~20, and batch size~128.

For student-prefix objectives, responses are instead sampled on-policy from the current student under the same prompt distribution. The teacher is then prefilled only to provide token-level supervision on the states visited by the student. Therefore, the key distinction between teacher-prefix and student-prefix training is not the underlying problem set, but the state distribution on which matching is enforced. For the RL follow-up stage, we use the same training dataset as in the distillation stage, reformatted for grouped on-policy rollouts and outcome-based accuracy reward computation.

\subsection{Standalone Distillation Hyperparameters}
\label{app:standalone-hparams}

We run standalone distillation under two sequence-length regimes: a short regime with a maximum generated length of 128 tokens, and a long regime with a maximum generated length of 4096 tokens. Unless otherwise noted, all four decoupled objectives share the following optimization hyperparameters:
\begin{itemize}
    \item bf16 training;
    \item learning rate $5\times 10^{-7}$;
    \item batch size 32;
    \item constant learning-rate schedule with 5\% warmup;
    \item 1000 optimizer steps;
    \item random seed 42;
    \item evaluation every 30 steps;
    \item checkpointing every 40 steps, with a save-total-limit of 10.
\end{itemize}

For on-policy student-prefix training and intermediate evaluation, we use colocated vLLM \citep{vllm} decoding with tensor parallel size 1. The GPU memory utilization is set conservatively to ensure stable rollout generation throughout training.

\subsection{Full-Vocabulary Matching}
\label{app:reverse-kl-objective}

Our reverse-KL and forward-KL implementations compute the loss over the full teacher and student vocabulary distributions, rather than restricting matching to the teacher's top-$k$ support. This is important because the behavior of KL-based objectives depends on the complete log-ratio geometry over the vocabulary.

\subsection{Math Evaluation During Training}
\label{app:math-eval-protocol}

During both standalone distillation and RL, we evaluate the model every 30 steps on a held-out math benchmark. The in-training evaluation callback uses the same vLLM-based generation pipeline across all methods. Unless otherwise stated, evaluation is performed with temperature~0.6 and a maximum generation budget of 4096 tokens.

\subsection{RL / GRPO Follow-Up Configuration}
\label{app:rlvr-details}

For the RL follow-up experiments, we use GRPO with an accuracy-only outcome reward. We do not include auxiliary format rewards in the primary comparison. Unless otherwise noted, the main GRPO settings are:
\begin{itemize}
    \item bf16 training;
    \item learning rate $10^{-6}$;
    \item batch size 32;
    \item group size 8;
    \item maximum response length 4096;
    \item 1000 RL update steps;
    \item constant learning-rate schedule with 5\% warmup;
    \item no explicit online teacher KL penalty (\(\beta=0\));
    \item one policy update per generated group.
\end{itemize}

For on-policy generation during GRPO, we use vLLM in colocated mode with tensor parallel size 1.

\subsection{Computational Resources}
\label{Computational_Resources}
Model training was conducted on a single NVIDIA H100 GPU with 94 GB of GPU memory. 
On-policy distillation took approximately 16 hours on average, while reinforcement learning (RL) took approximately 14 hours on average.

\section{Implementation Details}
\label{app:implementation}

\subsection{FLOPs Accounting for Teacher-Prefix and Student-Prefix Reverse KL}
\label{app:flops-accounting}

When comparing teacher-prefix and student-prefix training dynamics under a
compute-normalized x-axis, we distinguish two accounting conventions.
If teacher-prefix data are precomputed only as trajectories, then each
teacher-prefix update still requires a teacher forward pass to obtain the
teacher distribution.
If teacher logits are also precomputed and cached, the online training cost of
teacher-prefix reverse KL excludes this teacher forward pass.
The latter convention is useful for measuring the best-case computational
advantage of offline teacher-prefix training.

Let \(B\) be the batch size, \(P\) the average prompt length, \(R\) the response
length, and \(L=P+R\).
Let \(F_s(L)\) and \(F_t(L)\) denote student and teacher full-sequence forward
FLOPs on a batch of length \(L\), \(B_s(L)\) the student backward FLOPs, and
\(G_s(P,R)\) the student autoregressive generation FLOPs with KV cache.
We use the standard approximation \(B_s(L)\approx 2F_s(L)\).
Ignoring the comparatively small elementwise KL arithmetic, teacher-prefix
reverse KL with cached teacher logits costs
\[
C_{\mathrm{off}}^{\mathrm{cached}}(L)
\approx F_s(L)+B_s(L)
\approx 3F_s(L).
\]
Student-prefix reverse KL must additionally generate student rollouts online and
evaluate the teacher on those student-generated sequences:
\[
C_{\mathrm{on}}(P,R,L)
\approx G_s(P,R)+F_t(L)+F_s(L)+B_s(L)
\approx G_s(P,R)+F_t(L)+3F_s(L).
\]

For the Qwen3-0.6B student used in this work, the relevant configuration is
hidden size \(H=1024\), intermediate size \(I=3072\), number of layers
\(N=28\), number of attention heads \(16\), number of key-value heads \(8\),
head dimension \(128\), and vocabulary size \(V=151936\).
We count a multiply-add as two FLOPs.
For a single token, the approximate dense transformer cost is
\[
N\left[
2H(H+2H_{\mathrm{kv}})
+2H^2
+2\cdot 3HI
\right],
\]
where \(H_{\mathrm{kv}}=8\cdot128=1024\).
The three terms correspond to QKV projection, output projection, and SwiGLU MLP
projection.
This gives approximately \(0.763\) GFLOPs per token before the language-model
head.
Because our KL losses use full-vocabulary logits, the tied LM head contributes
\[
2HV = 2\cdot1024\cdot151936 \approx 0.311
\]
GFLOPs per token.
Thus the student forward cost is approximately \(1.075\) GFLOPs per token plus
the causal attention quadratic term.

For the 128-token setting, the measured average prompt length is
\(P\approx92\), and the teacher rollout length is effectively \(R\approx128\),
so \(L\approx220\) with batch size \(B=32\).
Under the above approximation, this yields
\[
F_s(L\approx220)\approx 7.73\ \mathrm{TFLOPs},
\qquad
B_s(L)\approx 15.46\ \mathrm{TFLOPs}.
\]
Therefore cached-logit teacher-prefix reverse KL costs
\[
C_{\mathrm{off}}^{\mathrm{cached}}
\approx 3F_s
\approx 23.19\ \mathrm{TFLOPs/step}.
\]

For student-prefix reverse KL, generation is counted using KV cache.
The generation cost consists of a prompt prefill plus cached decoding:
\[
G_s(P,R)
\approx \mathrm{prefill}(P)
+\sum_{r=1}^{R}\mathrm{decode}(P+r).
\]
With KV cache, this is comparable in FLOPs to a causal forward pass over the
final sequence length, although it is slower in wall-clock time because decoding
is sequential.
For \(P\approx92\) and \(R\approx128\), we estimate
\[
G_s(P,R)\approx 7.66\ \mathrm{TFLOPs}.
\]
For the Qwen3-4B teacher, using hidden size \(2560\), intermediate size \(9728\),
36 layers, 32 attention heads, 8 key-value heads, and the same vocabulary size,
we estimate
\[
F_t(L\approx220)\approx 53.13\ \mathrm{TFLOPs}.
\]
Hence
\[
C_{\mathrm{on}}
\approx 7.66 + 53.13 + 3\cdot 7.73
\approx 83.98\ \mathrm{TFLOPs/step}.
\]
The resulting compute ratio under cached teacher logits is
\[
\frac{C_{\mathrm{on}}}{C_{\mathrm{off}}^{\mathrm{cached}}}
\approx
\frac{83.98}{23.19}
\approx 3.62.
\]
Thus, under this accounting, one online student-prefix reverse-KL update costs
roughly \(3.6\) cached-logit teacher-prefix reverse-KL updates for the
128-token experiments.

\subsection{Fused Full-Vocabulary KL Kernel}
\label{app:fused-kl}

Full-vocabulary KL is required for our token-level objectives, but a naive
implementation is memory prohibitive in long-context distillation. If teacher
logits, student logits, probabilities, and tokenwise KL terms are explicitly
materialized, the dominant intermediate tensors scale as
$B\times L\times |\mathcal V|$. This is substantially more expensive than
top-$k$ teacher matching and quickly becomes the memory bottleneck.

We therefore implement a fused full-vocabulary KL kernel. The kernel streams over
the vocabulary dimension in tiles, computes the final projection for each tile,
updates the softmax normalizers with online log-sum-exp statistics, and
accumulates the KL contribution without storing full vocabulary-sized logits or
probability tensors. This is an exact reformulation of the same full-vocabulary
KL objective, not a top-$k$ or sampled approximation, and it does not change the
training data or loss. Following the same memory-saving principle as fused
large-vocabulary kernels such as Liger Kernel~\citep{ligerkernel}, the
implementation replaces $O(|\mathcal V|)$ materialized per-token intermediates
with tile-level computation and constant-size running statistics, making
long-context full-vocabulary KL practical for our experiments.

\begin{table}[t]
\caption{General multiple-choice evaluation after 128-token standalone
distillation with a Qwen3-4B teacher.
We report \texttt{acc} for MMLU and \texttt{acc\_norm} for ARC-Challenge,
HellaSwag, and PIQA; Avg. is the unweighted average across the four benchmarks.}
\label{tab:general_eval_128}
\centering
\small
\begin{tabular}{@{}lccccc@{}}
\toprule
Model & ARC-C & HellaSwag & MMLU & PIQA & Avg. \\
\midrule
Base & 44.88 & 53.50 & 52.45 & 70.13 & 55.24 \\
Off-policy + Forward KL & 45.73 & 53.49 & 51.87 & 69.91 & 55.25 \\
Off-policy + Reverse KL & 46.16 & 53.37 & 52.41 & 69.86 & 55.45 \\
On-policy + Forward KL & 46.33 & 53.15 & 51.82 & 70.08 & 55.34 \\
On-policy + Reverse KL & 46.25 & 53.27 & 52.25 & 70.08 & \textbf{55.46} \\
\bottomrule
\end{tabular}
\end{table}

\section{GRPO Follow-up Evaluation}
\label{app:grpo-eval}

Tables~\ref{tab:math_eval_128_grpo_full} and
\ref{tab:math_eval_4096_grpo_full} report the full mathematical reasoning
evaluation after GRPO initialized from the final 128-token and 4096-token
distillation checkpoints. We use the same evaluation protocol as in
Section~\ref{sec:setup_exp}: Avg@$k$ and Pass@$k$ use $k=3$ for GSM8K and
MATH500, and $k=5$ for AMC23 and AIME24. These results complement the training
dynamics in Section~\ref{sec:results} and show how different distillation
warmups affect post-RL performance across benchmarks.

\begin{table}[!t]
\caption{Math evaluation after GRPO initialized from 128-token distillation.
Len is the average response length.}
\label{tab:math_eval_128_grpo_full}
\centering
\scriptsize
\setlength{\tabcolsep}{2.6pt}
\resizebox{\textwidth}{!}{
\begin{tabular}{@{}lllcccccccccccc@{}}
\toprule
\multirow{2}{*}{Prefix} &
\multirow{2}{*}{Teacher} &
\multirow{2}{*}{KL} &
\multicolumn{3}{c}{GSM8K} &
\multicolumn{3}{c}{MATH500} &
\multicolumn{3}{c}{AMC23} &
\multicolumn{3}{c}{AIME24} \\
\cmidrule(lr){4-6}
\cmidrule(lr){7-9}
\cmidrule(lr){10-12}
\cmidrule(l){13-15}
& & & Avg@$k$ & Pass@$k$ & Len & Avg@$k$ & Pass@$k$ & Len & Avg@$k$ & Pass@$k$ & Len & Avg@$k$ & Pass@$k$ & Len \\
\midrule
\multirow{4}{*}{Student}
& \multirow{2}{*}{Qwen3-4B}
& Forward & 67.80 & 81.65 & 589 & 42.81 & 55.18 & 2299 & 22.00 & 40.00 & 4094 & 3.33 & 6.67 & 6000 \\
& & Reverse & 66.46 & 78.39 & 398 & 43.53 & 55.34 & 2165 & 26.00 & 47.50 & 3220 & 0.67 & 3.33 & 5701 \\
\cmidrule(lr){2-15}
& \multirow{2}{*}{Qwen3-8B}
& Forward & 66.14 & 80.29 & 1375 & 38.68 & 51.34 & 3293 & 16.50 & 37.50 & 5442 & 0.67 & 3.33 & 6576 \\
& & Reverse & 65.13 & 78.77 & 1565 & 40.63 & 52.50 & 3350 & 21.50 & 47.50 & 4515 & 3.33 & 6.67 & 5902 \\
\midrule
\multirow{4}{*}{Teacher}
& \multirow{2}{*}{Qwen3-4B}
& Forward & 68.71 & 83.62 & 820 & 42.75 & 54.94 & 2613 & 26.00 & 52.50 & 4960 & 2.00 & 6.67 & 7122 \\
& & Reverse & 66.09 & 79.30 & 574 & 41.68 & 54.20 & 2207 & 20.50 & 42.50 & 3846 & 0.00 & 0.00 & 5423 \\
\cmidrule(lr){2-15}
& \multirow{2}{*}{Qwen3-8B}
& Forward & 66.52 & 81.05 & 658 & 42.29 & 55.80 & 2131 & 21.50 & 47.50 & 3528 & 0.67 & 3.33 & 6030 \\
& & Reverse & 66.34 & 80.59 & 620 & 40.22 & 53.06 & 2214 & 17.50 & 40.00 & 3990 & 2.00 & 3.33 & 5932 \\
\bottomrule
\end{tabular}
}
\end{table}

\begin{table}[!t]
\caption{Math evaluation after GRPO initialized from 4096-token distillation.
Len is the average response length.}
\label{tab:math_eval_4096_grpo_full}
\centering
\scriptsize
\setlength{\tabcolsep}{2.6pt}
\resizebox{\textwidth}{!}{
\begin{tabular}{@{}lllcccccccccccc@{}}
\toprule
\multirow{2}{*}{Prefix} &
\multirow{2}{*}{Teacher} &
\multirow{2}{*}{KL} &
\multicolumn{3}{c}{GSM8K} &
\multicolumn{3}{c}{MATH500} &
\multicolumn{3}{c}{AMC23} &
\multicolumn{3}{c}{AIME24} \\
\cmidrule(lr){4-6}
\cmidrule(lr){7-9}
\cmidrule(lr){10-12}
\cmidrule(l){13-15}
& & & Avg@$k$ & Pass@$k$ & Len & Avg@$k$ & Pass@$k$ & Len & Avg@$k$ & Pass@$k$ & Len & Avg@$k$ & Pass@$k$ & Len \\
\midrule
\multirow{4}{*}{Student}
& \multirow{2}{*}{Qwen3-4B}
& Forward & 73.01 & 85.14 & 6335 & 47.61 & 59.64 & 7764 & 29.00 & 47.50 & 7917 & 4.00 & 10.00 & 7977 \\
& & Reverse & 61.51 & 76.50 & 2046 & 39.93 & 52.24 & 2993 & 21.00 & 42.50 & 4558 & 0.00 & 0.00 & 5815 \\
\cmidrule(lr){2-15}
& \multirow{2}{*}{Qwen3-8B}
& Forward & 59.54 & 76.04 & 8192 & 37.27 & 52.04 & 8192 & 21.00 & 40.00 & 8192 & 5.33 & 6.67 & 8192 \\
& & Reverse & 61.97 & 76.57 & 562 & 36.32 & 49.50 & 3146 & 16.00 & 27.50 & 5303 & 0.67 & 3.33 & 6917 \\
\midrule
\multirow{4}{*}{Teacher}
& \multirow{2}{*}{Qwen3-4B}
& Forward & 65.18 & 78.09 & 1021 & 40.55 & 54.24 & 2179 & 25.00 & 50.00 & 3419 & 0.67 & 3.33 & 5078 \\
& & Reverse & 62.07 & 77.41 & 2328 & 35.11 & 48.38 & 4632 & 14.00 & 32.50 & 5806 & 0.67 & 3.33 & 7210 \\
\cmidrule(lr){2-15}
& \multirow{2}{*}{Qwen3-8B}
& Forward & 67.45 & 80.44 & 643 & 41.63 & 54.76 & 1784 & 24.50 & 42.50 & 2662 & 2.67 & 6.67 & 3473 \\
& & Reverse & 68.87 & 82.26 & 5623 & 41.63 & 55.06 & 7426 & 20.50 & 40.00 & 7716 & 0.67 & 3.33 & 7727 \\
\bottomrule
\end{tabular}
}
\end{table}

\section{General-Domain Evaluation}
\label{app:general-eval-protocol}

To test whether math-domain distillation preserves broad language-model capabilities, we evaluate the distilled models with \texttt{lm-evaluation-harness} \citep{eval-harness} using the Hugging Face backend. We set \texttt{trust\_remote\_code=True}, \texttt{dtype=bfloat16}, and \texttt{batch\_size=auto}, and evaluate all models under the same multiple-choice scoring protocol. The benchmarks and few-shot settings are: MMLU with 5-shot prompting, reported with \texttt{acc}; ARC-Challenge with 25-shot prompting, reported with \texttt{acc\_norm}; HellaSwag with 10-shot prompting, reported with \texttt{acc\_norm}; and PIQA with 0-shot prompting, reported with \texttt{acc\_norm}.

As shown in Table~\ref{tab:general_eval_128}, this evaluation is conducted in the 128-token distillation setting, using Qwen3-4B-Base as the teacher and Qwen3-0.6B-Base as the student. All four distilled models improve the unweighted average over the base student model, with on-policy models achieving stronger averages under both forward and reverse KL. This trend is consistent with \citet{chen2025retainingdoingroleonpolicy}, who argue that on-policy data can help mitigate forgetting.

\section{Limitations and Broader Impact}
\label{sec:limitation_impact}
\paragraph{Limitations.} Our study focuses on mathematical reasoning, leaving open how the four decoupled objectives behave on other domains such as code generation or visual understanding. Our RL follow-up uses GRPO with an accuracy-only outcome reward; how the four objectives interact with other RL algorithms is left for future investigation.

\paragraph{Broader Impact.} By clarifying how KL direction and prefix source jointly shape distillation quality and downstream RL trainability, our analysis can help practitioners build smaller reasoning models more reliably and at lower compute cost, broadening access to capable open-weight reasoning systems. Since our method transfers behavior from a fixed teacher, it inherits but does not amplify the general risks of LLMs, and standard mitigations such as safety evaluation and responsible release practices remain applicable.


\end{document}